\documentclass{article}

\usepackage{microtype}
\usepackage{graphicx}
\usepackage{subfigure}
\usepackage{booktabs}
\usepackage{hyperref}
\usepackage[utf8]{inputenc} 
\usepackage[T1]{fontenc}    
\usepackage{url}            
\usepackage{xspace}
\usepackage{wrapfig}
\usepackage{tikz}
\usetikzlibrary{positioning}
\usetikzlibrary{arrows, arrows.meta}
\usepackage{amsthm}
\newtheorem{theorem}{Theorem}

\usepackage{amsmath,amsfonts,bm}

\newcommand{\BoundedRegion}{\mathcal{B}}
\newcommand{\Intensity}{\lambda}

\newcommand{\Dim}{d}
\newcommand{\SetSize}{n}

\newcommand{\Set}{X}
\newcommand{\SetElement}{\vx}
\newcommand{\ZSet}{Z}

\newcommand{\Poisson}{\mathrm{Poisson}}

\def\Figref#1{Figure~\ref{#1}}

\def\Secref#1{Section~\ref{#1}}

\def\eqref#1{Eq.~\ref{#1}}
\def\Eqref#1{Equation~\ref{#1}}

\def\1{\bm{1}}

\def\vg{{\bm{g}}}
\def\vh{{\bm{h}}}

\def\vv{{\bm{v}}}

\def\vx{{\bm{x}}}

\def\vz{{\bm{z}}}

\def\mI{{\bm{I}}}

\DeclareMathAlphabet{\mathsfit}{\encodingdefault}{\sfdefault}{m}{sl}
\SetMathAlphabet{\mathsfit}{bold}{\encodingdefault}{\sfdefault}{bx}{n}

\newcommand{\E}{\mathbb{E}}

\newcommand{\R}{\mathbb{R}}

\newcommand{\softmax}{\mathrm{softmax}}

\newcommand{\KL}{D_{\mathrm{KL}}}

\DeclareMathOperator{\Tr}{Tr}

\usepackage[accepted]{icml2021}

\icmltitlerunning{Scalable Normalizing Flows for Permutation Invariant Densities}

\begin{document}

\twocolumn[
\icmltitle{Scalable Normalizing Flows for Permutation Invariant Densities}

\begin{icmlauthorlist}
\icmlauthor{Marin Biloš}{tum}
\icmlauthor{Stephan Günnemann}{tum}
\end{icmlauthorlist}

\icmlaffiliation{tum}{Technical University of Munich, Germany}

\icmlcorrespondingauthor{}{bilos@in.tum.de}

\icmlkeywords{normalizing flows,permutation invariance,permutation equivariance,continuous normalizing flows,neural ode}

\vskip 0.3in
]

\printAffiliationsAndNotice{}

\begin{abstract}
    Modeling sets is an important problem in machine learning since this type of data can be found in many domains. A promising approach defines a family of permutation invariant densities with continuous normalizing flows. This allows us to maximize the likelihood directly and sample new realizations with ease. In this work, we demonstrate how calculating the trace, a crucial step in this method, raises issues that occur both during training and inference, limiting its practicality. We propose an alternative way of defining permutation equivariant transformations that give closed form trace. This leads not only to improvements while training, but also to better final performance. We demonstrate the benefits of our approach on point processes and general set modeling.
\end{abstract}

\section{Introduction}\label{sec:introduction}

We say that data is exchangeable if it lacks ordering, i.e., any permutation of random variables has the same probability. Examples include point clouds, items in a shopping cart, tracking household electricity consumption in a city etc. Often, the number of elements is also a random variable, like in locations of cellular stations or trees in a forest. We wish to find a generative model that best describes such data. To do that, we need to specify a symmetric density --- one that is invariant to data permutations.

A powerful framework to model densities are normalizing flows \cite{kobyzev2020normalizing}. They parametrize a transformation from a simple base distribution into a complex one. To define a permutation invariant density, it is enough to have a permutation invariant base distribution, and a permutation equivariant transformation \cite{papamakarios2019normalizing}. The latter is impossible to obtain with most of the traditional architectures \cite{dinh2016density,kingma2016improved}, unless we assume the points are independent of each other.

However, the i.i.d.\ assumption is a strong limitation because the presence of one object can influence the distribution of the others. For example, short trees grow near each other, but are inhibited by taller trees \citep{ogata1985estimation}. We want to find a way to model invariant densities that capture this behavior. A special class of flows, called continuous normalizing flows \cite{chen2018neural}, support unrestricted transformations that allow us to directly use permutation equivariant layers \cite{zaheer2017deep,lee2019set}. The general approach is illustrated in \Figref{fig:example}.

\begin{figure}
    \centering
    \begin{tikzpicture}
        \node[inner sep=0pt] at (0, 0)
        {\includegraphics[height=2.5cm]{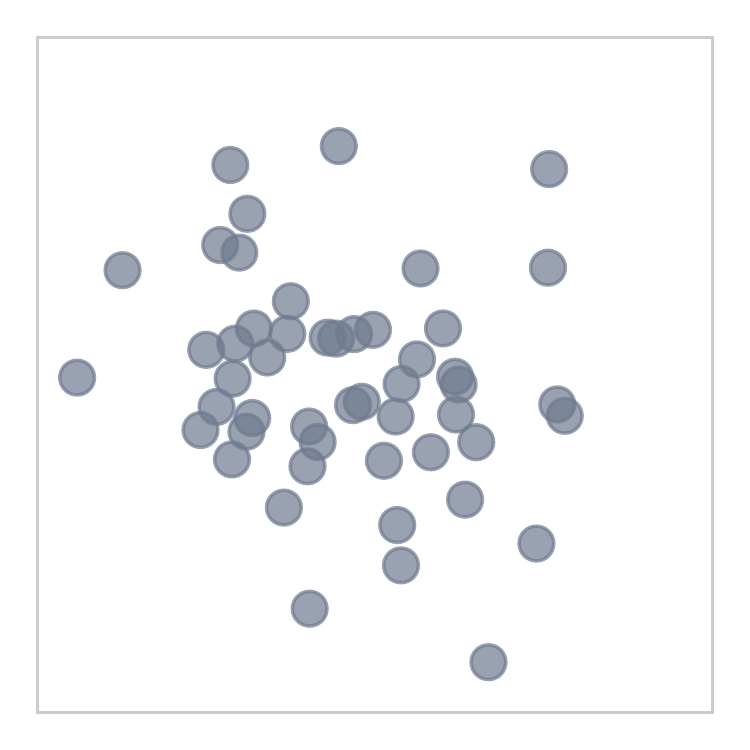}};

        \node[inner sep=0pt] at (1.8, 0)
        {\includegraphics[height=2.5cm]{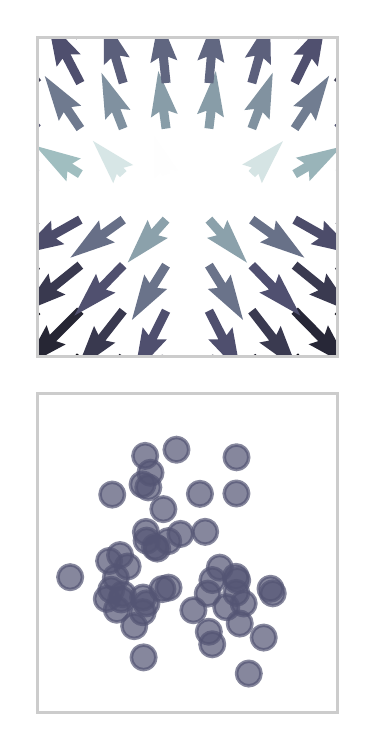}};

        \node[inner sep=0pt] at (3, 0)
        {\includegraphics[height=2.5cm]{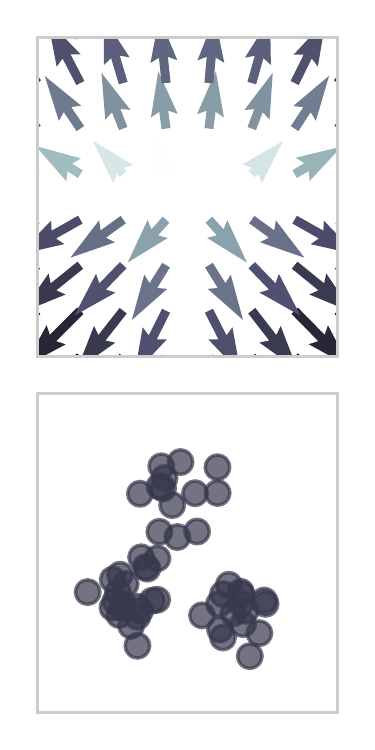}};

        \node[inner sep=0pt] at (4.8, 0)
        {\includegraphics[height=2.5cm]{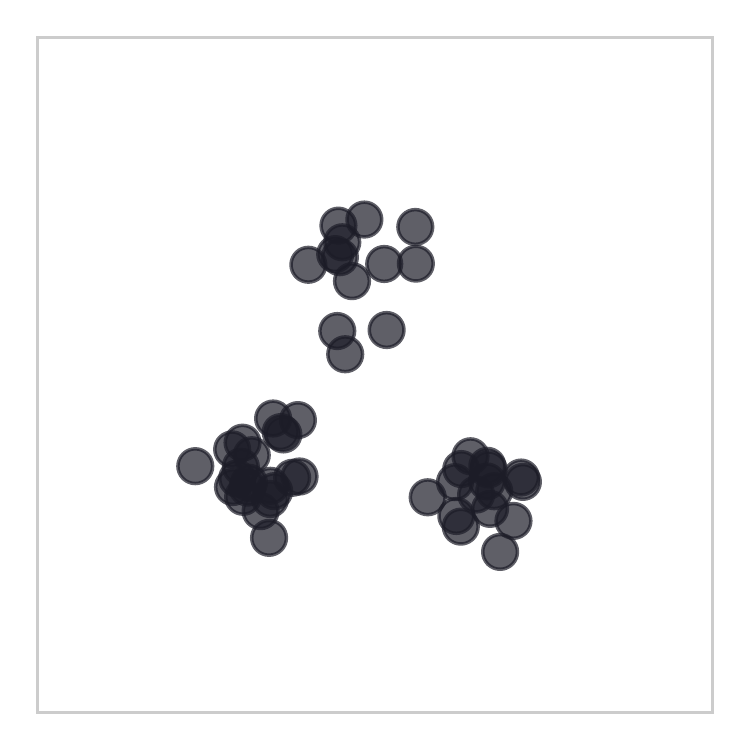}};

        \node at (0, -1.8) {\small Base distribution};
        \node at (4.8, -1.8) {\small Target distribution};
        \node at (2.4, -1.8) {\small $f$};

        \node[outer sep=0.2cm] (t0) at (0.0, -1.4) {\small $t_0$};
        \node[outer sep=0.2cm] (t1) at (4.9, -1.4) {\small $t_1$};

        \draw[triangle 45-triangle 45,thick,black!50] (t0) -- (t1);
        \draw [thick,black!50] (1.8, -1.5) -- (1.8, -1.3);
        \draw [thick,black!50] (3, -1.5) -- (3, -1.3);
    \end{tikzpicture}
    \vspace*{-0.3cm}
    \caption{Illustration of our approach. Transforming random points from the permutation invariant base distribution with an ODE gives a realization in the target distribution. The opposite direction calculates the likelihood of the observed sample. The dynamics $f$ has to be a permutation equivariant neural network.}
    \label{fig:example}
    \vspace*{-0.5cm}
\end{figure}
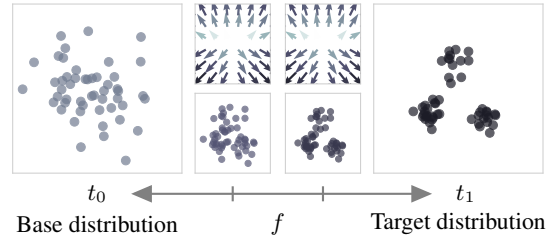

The main disadvantage of continuous flows is the requirement to calculate the trace of the Jacobian of the transformation. In practice, this is done using an unbiased estimator. In this work, we show that this introduces additional noise during training, making it unstable. Further, for inference and in some downstream tasks (e.g., out-of-distribution detection), we would still have to use costly exact trace computation.

To combat this, we propose a straightforward decoupling of permutation equivariant functions to obtain a constant trace, i.e., we construct a volume preserving flow. We further extend this to have exact trace computation while keeping the original expressiveness. In the experiments we show that our method speeds-up the training, makes it more stable, provides better results, and scales with big datasets.

Additionally, we introduce novel models for point processes that account for complex interactions between the points and have closed form likelihood with easy sampling. Previous work resorted to less expressive models or using pseudolikelihood \cite{besag1975statistical}. We tackle other use cases as well, such as modeling point clouds.

\section{Background}\label{sec:background}

The realizations of a finite point process on a bounded region $\BoundedRegion \subset \R^\Dim$ are finite sets of points $\Set = \{\SetElement_1, \dots, \SetElement_\SetSize\}$, $\SetElement_i \in \BoundedRegion$. To fully specify a point process we need two distributions: the distribution of the number of points $p(\SetSize)$, and the permutation invariant probability density $\tilde{p}(\SetElement_1, \dots, \SetElement_\SetSize)$ of their locations. The probability of seeing $n$ points at locations $\Set$ is \cite{daleyintroduction}:
\begin{align}\label{eq:point_process_likelihood}
    p(\Set) = \SetSize! p(\SetSize) \tilde{p}(\SetElement_1, \dots, \SetElement_\SetSize) .
\end{align}
Traditionally, point processes are defined with an intensity function $\Intensity(\cdot)$ that measures how many points we expect to see on some subset. This is equivalent to what we described in \eqref{eq:point_process_likelihood}, therefore, we can safely use our parametrization without losing generality.

Perhaps the simplest non-trivial point process model is an inhomogeneous Poisson process. The number of points $\SetSize$ follows a Poisson distribution and the locations of the points are assumed to be generated i.i.d.\ from some density $\tilde{p}(\Set) = \prod_i p(\SetElement_i)$ \citep{chiu2013stochastic}. To generate a new realization, we first sample $\SetSize \sim \Poisson(\Intensity)$, and then sample $\SetSize$ points $\SetElement \sim p(\SetElement)$. However, IHP alone cannot capture more complex permutation invariant densities, with interactions between the points.

A solution candidate are normalizing flows since they define complex distributions with invertible smooth transformations of the initial random variable. That means, if we apply a function $f : \BoundedRegion^\Dim \rightarrow \BoundedRegion^\Dim$ to a random variable $\vz \sim q(\vz)$, where $f$ is invertible and differentiable, we can get the density of $\vx = f(\vz)$ by calculating the change of variables formula: $p(\vx) = q(\vz) \left| \det J_f(\vz) \right |^{-1}$, where $J$ is the Jacobian of $f$. The main challenge is defining a function that can be efficiently inverted and whose determinant is not prohibitively expensive to calculate. Thus, most approaches design transformations that have a triangular Jacobian, using coupling \cite{dinh2016density} or autoregressive transformations \cite{kingma2016improved}. This, however, permits us from using permutation equivariant functions since having rich interactions between the elements requires a \textit{full} Jacobian.

Another way to transform an initial sample $\vz$ is with an ordinary differential equation $f(\vz(t), t) = \partial \vz(t) / \partial t$. The instantaneous change in log-density is \citep{chen2018neural}:
\begin{align}\label{eq:instant_change_of_variables}
    \frac{\partial  \log p(\vz(t))}{\partial t} = -\Tr\left( \frac{\partial f}{\partial \vz(t)} \right),
\end{align}
and the final log-density is obtained by integrating across time, e.g., using a black-box ODE solver:
\begin{align}\label{eq:instant_change_of_variables_density}
    \log p(\vx) = \log q(\vz(t_0)) - \int_{t_0}^{t_1} \Tr\left( \frac{\partial f}{\partial \vz(t)} \right) dt .
\end{align}
It is easy to show that if $f$ is permutation equivariant and the base distribution $q$ is permutation invariant, the final density $p$ is also permutation invariant (see Supplementary Material). The only other constraint is that $f$, along with its derivatives,  needs to be Lipschitz continuous which is satisfied by using smooth activations in a neural network. The benefit of this approach is that we are restricting $f$ only mildly, meaning the dimensions can interact with each other arbitrarily. We utilize this to model interactions between the points in a point process.

\section{Model}\label{sec:model}

A general way to define an equivariant map $f : \BoundedRegion^\SetSize \rightarrow \BoundedRegion^\SetSize$ is to define an invariant function $g : \BoundedRegion^\SetSize \rightarrow \R$, and let $f(\Set) = \nabla_\Set g(\Set)$ \citep{papamakarios2019normalizing}. However, this is computationally costly and numerically unstable \citep{kohler2020equivariant}. We can instead construct an equivariant (deep set) layer directly \citep{zaheer2017deep,maron2020learning}:
\begin{align}\label{eq:deepset}
    f(\Set)_i = g(\SetElement_i) + \sum_{j \ne i} h(\SetElement_j),
\end{align}
where $g, h : \BoundedRegion \rightarrow \BoundedRegion$ are neural networks.

An alternative equivariant mapping is a multihead self-attention \citep{vaswani2017attention}. In general, the attention layer is defined for matrices $Q \in \R^{\SetSize \times d_k}, K \in \R^{\SetSize \times d_k}$ and $V \in \R^{\SetSize \times d_v}$ as follows:
\begin{align}\label{eq:attention-general}
    \mathrm{Attention}(Q, K, V) = \softmax \left( \frac{Q K^T}{\sqrt{d_k}} \right) V.
\end{align}
Our transformation is then:
\begin{align}\label{eq:attention}
    f(\Set)_i = \mathrm{Attention}(f_q(\Set), f_k(\Set), f_v(\Set))_i
\end{align}
where we treat $\Set$ as a stacked matrix of $\SetElement_i$ set elements. Functions $f_q, f_k$ and $f_v$ are permutation invariant neural networks, i.e., transform each point independently.

\subsection{Efficient trace calculation}

Computing the trace in \eqref{eq:instant_change_of_variables} during training is expensive since it scales quadratically $\mathcal{O}(\SetSize^2 \Dim^2)$ with the number of points and the dimension of $\BoundedRegion$. \citet{grathwohl2018ffjord} propose using Hutchinson's trace estimator which reduces the cost to $\mathcal{O}(\SetSize \Dim)$. However, this introduces new problems: 1) the estimator brings noise into training making it unstable; and 2) during inference, we still need to evaluate the exact trace to get the correct density. In the following we introduce a way to have closed form trace in linear time $\mathcal{O}(\SetSize \Dim)$, while retaining the expressiveness of layers in \eqref{eq:deepset} and \ref{eq:attention}. We do that by decoupling the original transformation into different nested parts, that give us exact trace computation (\Figref{fig:jacobian}).

\textbf{Interactions between points.} If we take a closer look at  \eqref{eq:deepset}, we can see that it constitutes of two decoupled functions, one is acting on the current element: $g(\SetElement_i)$, and the second corresponds to the contribution from all the other elements: $\sum_{j \ne i} h(\SetElement_j)$. If we ignore the first part, the trace will be always equal to $0$, meaning we do not have to evaluate \eqref{eq:instant_change_of_variables} since it is constant. This gives us a volume-preserving normalizing flow, similar to \cite{dinh2014nice,rezende2019equivariant}. Here, every point $\SetElement_i$ is transformed based on all the other points $\SetElement_{j \ne i}$. This is illustrated in \Figref{fig:jacobian} as a \textit{between points} transformation.

In practice, we efficiently implement $\vh_i = \sum_{j \ne i} h(\SetElement_j)$ for the $i$th element by calculating the sum of all $h(\SetElement_j)$ and then subtract $h(\SetElement_i)$. If instead of the sum we use mean aggregation, we divide everything by $\SetSize - 1$. We get a constant trace for the operation of finding the maximum value by returning the largest value for all elements that are not the actual maximum, i.e., with the usual max-pooling. At the true maximum we return the second largest value.

\textbf{Interactions within point.} The next question is if we can keep the function $g$ inside \eqref{eq:deepset} while retaining the fixed trace property. To do so, we apply the same trick as before --- separate the part that produces the constant trace from the rest. This is visualized in \Figref{fig:jacobian} as a \textit{within point} transformation. Each dimension in one element $\SetElement_i$ should be transformed using other dimensions but should not depend on itself. Since $g$ does not cary any invariance constraints, we can use arbitrary neural network layers. We implement the transformation $g$ by masking the weights, in particular, we use the existing MADE architecture \cite{germain2015made}. Such masked networks ensure that there is no computation path from the input dimension $\vx_{ij}$ to the output $g(\vx_i)_j$. The resulting Jacobian of $g$ always has zero values on the diagonal. The output dimension can be larger then the input, with groups of dimensions that are independent of $\vx_{ij}$. This allows us to build multilayer networks and keeps the constant trace requirement fulfilled.

\textbf{Per dimension transformation.} So far we got a transformation with a Jacobian matrix that has zeros on the diagonal --- a volume preserving flow. Finally, to get the full Jacobian, the last remaining part is the \textit{per dimension} transformation. In the simplest case, this is an $\R \rightarrow \R$ mapping, and can be easily calculated in closed form. A more expressive function, one that recovers $g$ fully, is $g(\vx_{ij}) = \tau(\vx_{ij}, \vg_i)$, where $\tau : \R^{d_g + 1} \rightarrow \R$ and $\vg_i$ does not depend on $\vx_{ij}$, i.e., it is a result of masked layers from above. This kind of decoupling was first introduced by \citet{chen2019neural}. By exploiting the properties of modern deep learning frameworks, the trace can now be calculated cheaply. We can extend this idea by also including the result of the between point interactions $\vh_i$.

The complete transformation of a single dimension of one set element $\SetElement_i$ is $f(\vx_{ij}) = \tau(\vx_{ij}, \vg_i, \vh_i)$. The function $\tau$ transforms the $j$th dimension of the $i$th element, conditioned on the within point interactions $\vg_i$, and the between points interactions $\vh_i$. To capture the original function $g$ exactly, the vector $\vg_i$ has to have a dimension of $\Dim - 1$; however, in practice, smaller dimensions can be used with good results and better efficiency. Function $\tau$ can be any neural network. We can detach the zero trace networks that produce $\vg_i$ and $\vh_i$ completely from the computation graph and just calculate $\partial \tau_{ij} / \partial x_{ij}$ in parallel to obtain the exact trace. When backpropagating we need to attach the networks back to allow learning of the parameters. More details can be found in \citet{chen2019neural}.

Alternatively, we can keep the original function $g$, and use stochastic trace estimation. That is, we only use the constant trace for between points interaction, while $g$ is unrestricted. The Jacobian produced by $g$ is a block diagonal matrix (see \Figref{fig:jacobian}). Therefore, we can estimate the trace as a sum of traces in these $\Dim \times \Dim$ blocks. This results in the reduced variance of the estimator. \citet{chen2020neural} use a similar approach for the attention-based models, in which they detach the gradient connections between certain elements. We can even resort to the exact calculation with the $\mathcal{O}(\SetSize \Dim^2)$ cost, which can be acceptable for $d \ll n$.

\begin{figure}
    \centering
    \begin{tikzpicture}
        \node[inner sep=0pt] at (-1.3, 0.5) {\small $\SetElement_1 \Bigg\{$};
        \node[inner sep=0pt] at (-1.3, -0.5) {\small $\SetElement_2 \Bigg\{$};
        \node[inner sep=0pt] at (-1.3, -2) {\small $J_f$};

        \node[inner sep=0pt] at (0, 0) {\includegraphics[height=1.8cm]{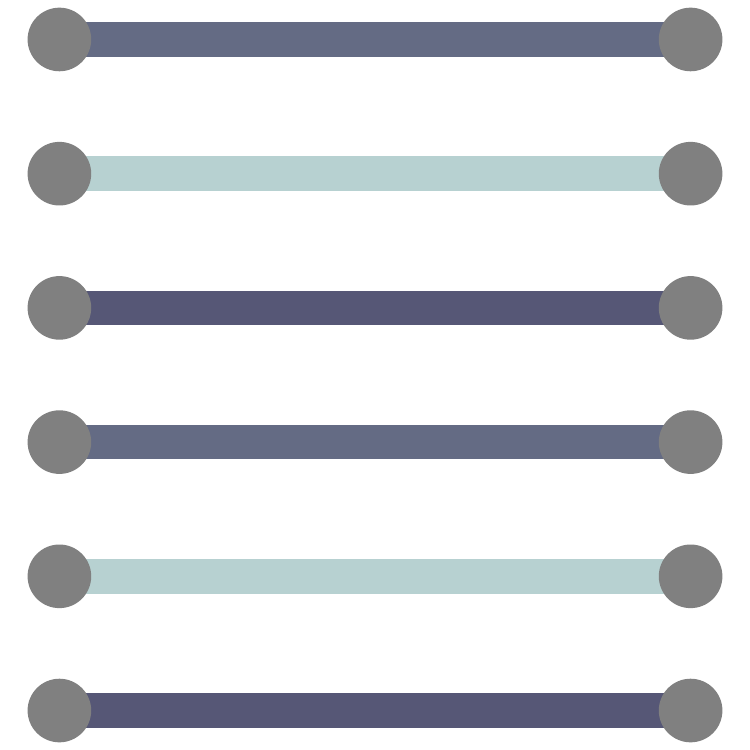}};
        \node[inner sep=0pt] at (0, -2) {\includegraphics[height=1.8cm]{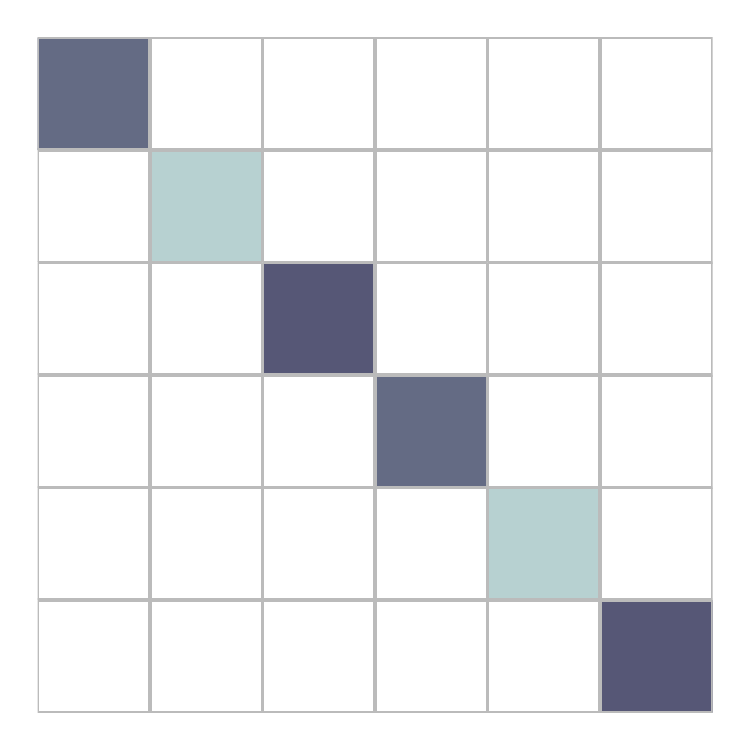}};
        \node[inner sep=0pt] at (0, 1.24) {\scriptsize Per dimension};

        \begin{scope}[xshift=1.9cm,on grid]
        \node[inner sep=0pt] at (0, 0) {\includegraphics[height=1.8cm]{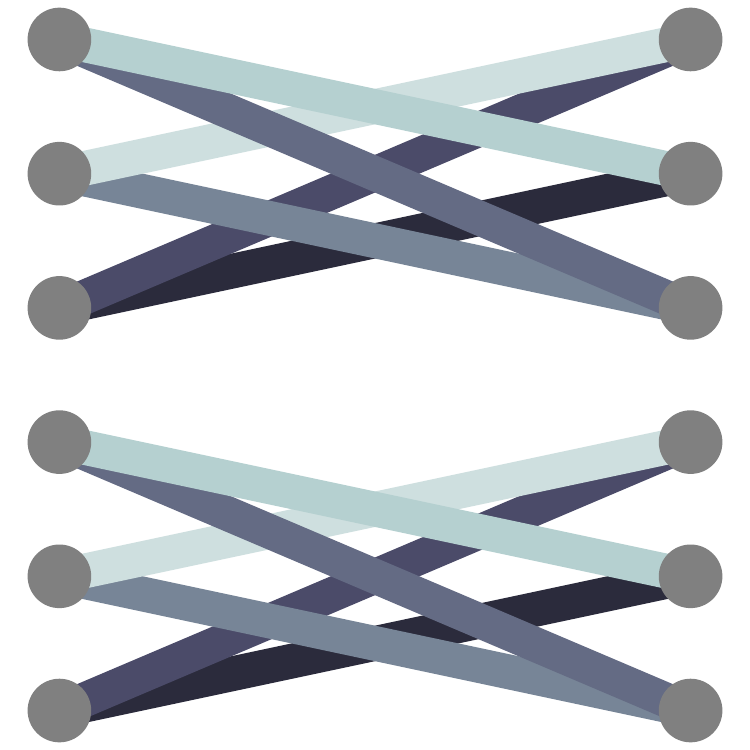}};
        \node[inner sep=0pt] at (0, -2) {\includegraphics[height=1.8cm]{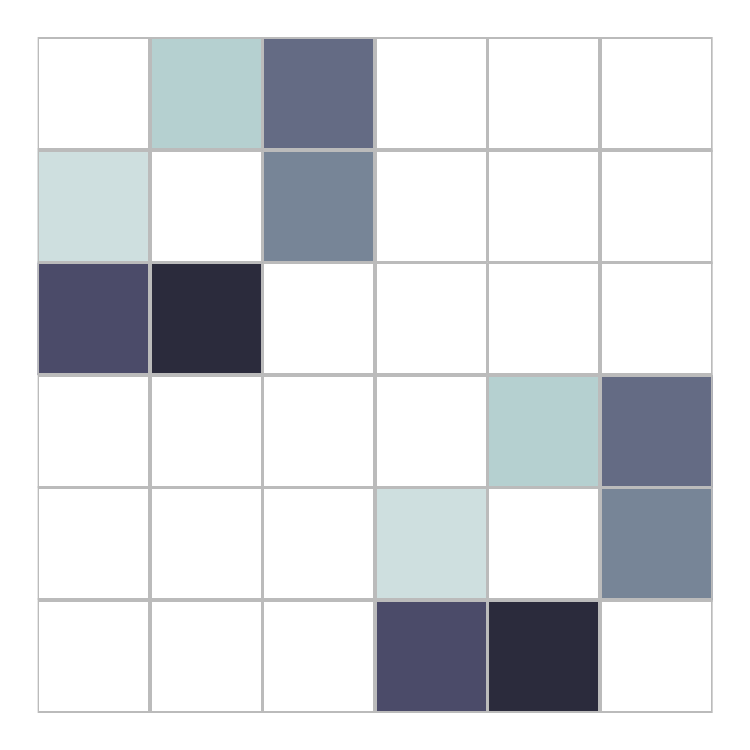}};
        \node[inner sep=0pt] at (0, 1.2) {\scriptsize Within point};
        \end{scope}

        \begin{scope}[xshift=3.8cm,on grid]
        \node[inner sep=0pt] at (0, 0) {\includegraphics[height=1.8cm]{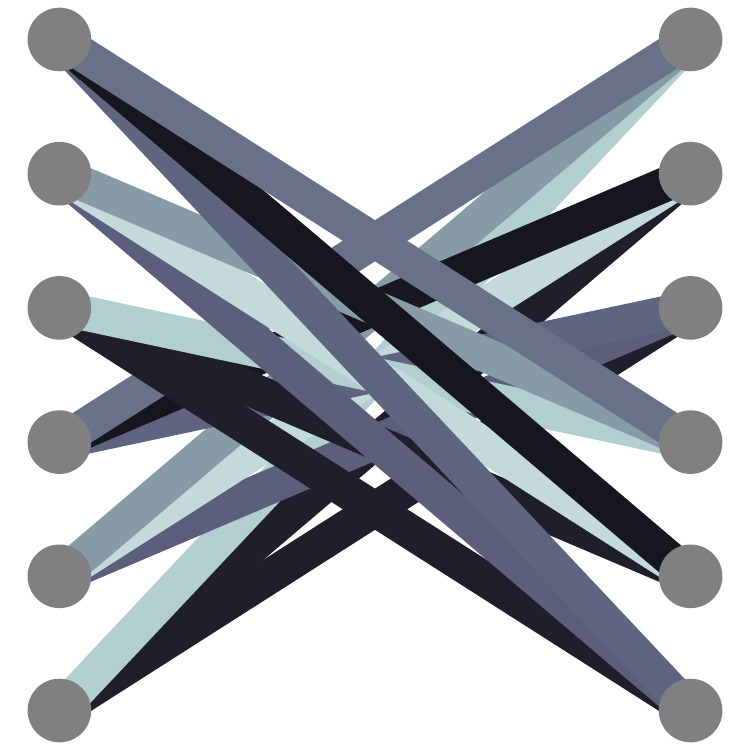}};
        \node[inner sep=0pt] at (0, -2) {\includegraphics[height=1.8cm]{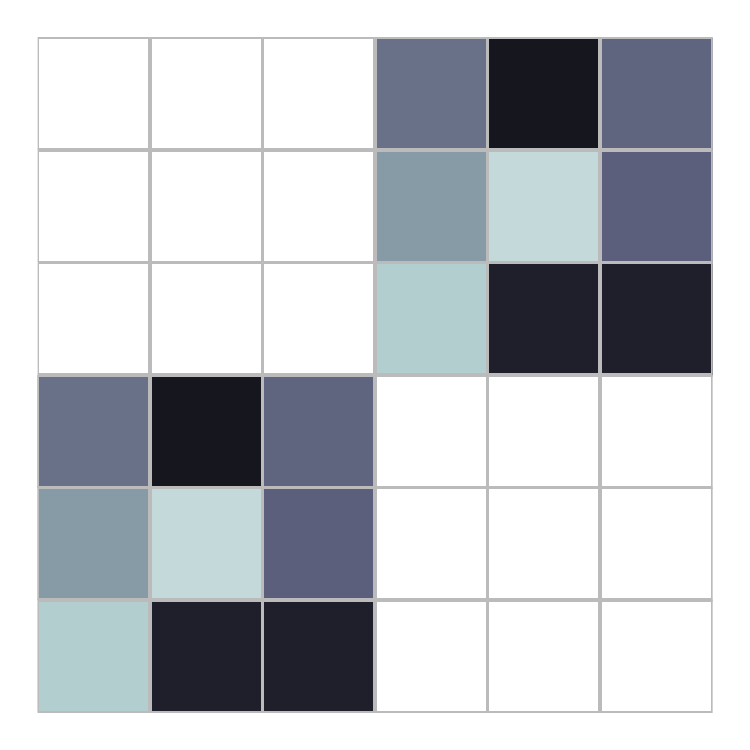}};
        \node[inner sep=0pt] at (0, 1.2) {\scriptsize Between points};
        \end{scope}

        \begin{scope}[xshift=5.7cm,on grid]
        \node[inner sep=0pt] at (0, 0) {\includegraphics[height=1.8cm]{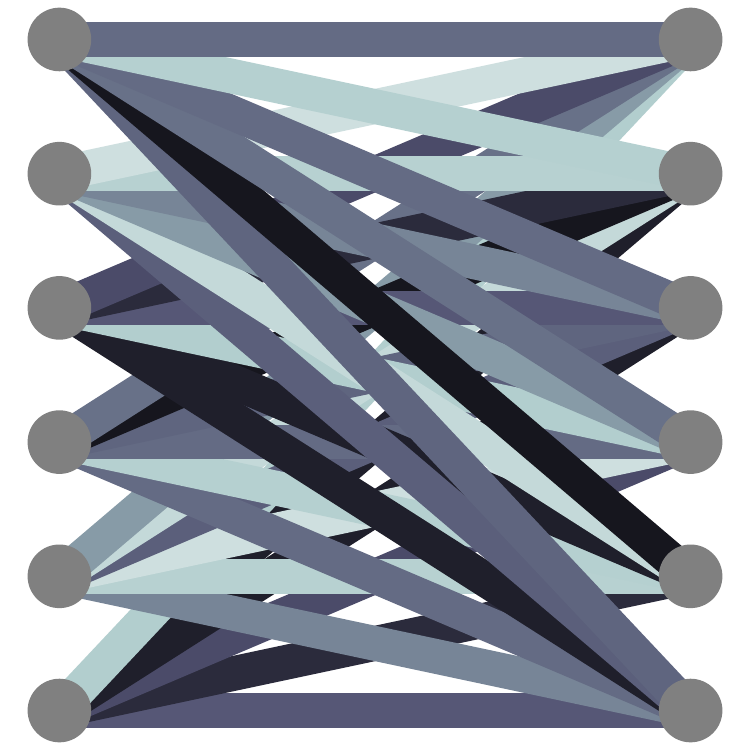}};
        \node[inner sep=0pt] at (0, -2) {\includegraphics[height=1.8cm]{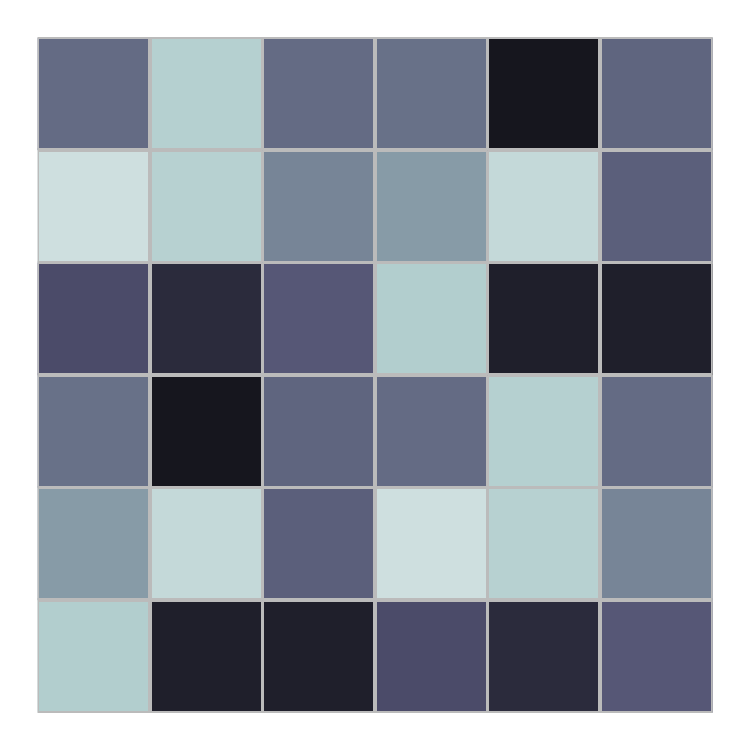}};
        \node[inner sep=0pt] at (0, 1.24) {\scriptsize $f$};
        \end{scope}

    \end{tikzpicture}
    \vspace{-0.5cm}
    \caption{Illustration of different interactions and the corresponding Jacobian decoupling for equivariant $f$ and inputs $\SetElement_1, \SetElement_2 \in \R^3$.}
    \label{fig:jacobian}
    \vspace{-0.5cm}
\end{figure}
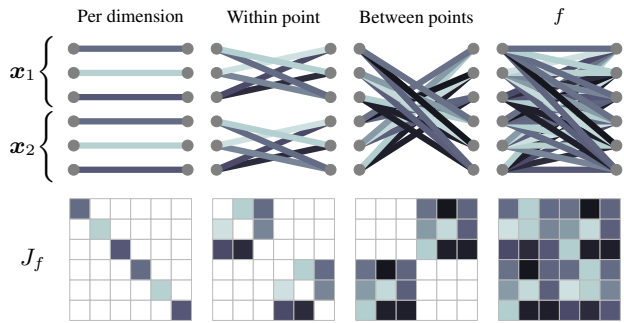

\textbf{Attention.} In the previous part we described how to implement the exact trace for \eqref{eq:deepset}. The attention layer (\eqref{eq:attention}) presents more challenges since it introduces the similarity measure in form of a dot product. The final attention weights for the $i$th element is the $i$th row of the $\softmax(QK^T)$ matrix. If we want a zero diagonal Jacobian, we need to obtain the elements of $Q$ with a MADE network $f_q(\SetElement_i)$. We also need to mask the diagonal of $QK^T$ with $-\infty$ to get zeros there after the $\softmax$ operation. This ensures that we never take the dot product of an element with itself. Therefore, the function $f_k(\cdot)$ can be an unrestricted network. At last, the output of an attention layer is a weighted sum of $\vv_i$. Since we have already zeroed-out diagonal entries in the similarity matrix, the function $f_v(\cdot)$ can be unrestricted too. This gives constant trace for the \textit{within} and \textit{between points} interaction. The \textit{per dimension} mapping can be implemented as discussed above. We experimentally show that our implementation does not hinder performance.

\textbf{Discussion.} The key to building an exact trace model is designing functions $g$ and $h$ --- interaction of one dimension with all the others, and interaction of points between themselves --- with zero trace Jacobian. The novelty of our approach is utilizing the properties of permutation equivariant functions. We propose exact trace models that can be either: 1) used as volume preserving flows, 2) combined with trace estimation or exact trace calculation, 3) extended within the framework of Jacobian decoupling \cite{chen2019neural}. In the experiments, we mainly use the third option. In the ablation studies we also explore the capabilities of the volume preserving flow.

\subsection{Architecture}

Our model is a continuous normalizing flow (CNF) that constitutes of a permutation invariant base distribution and a permutation equivariant function $f$ that defines the ordinary differential equation dynamics. A CNF can be seen as a single transformation: we can combine it with other layers (fixed or learnable) or stack multiple CNFs. We use a normal distribution $\mathcal{N}(\bm{0}, \mI)$ for our base distribution. If we require a distribution on a bounded domain, e.g., uniform $\mathcal{U}(\bm{0}, \bm{1})$, we can go to the $\R^{\SetSize \times \Dim}$ space with logit function, so that we can transform points with our CNF. We then map everything back to the original space with a sigmoid function.

The model is trained by maximizing the log-likelihood $p(\Set)$. Following the previous works, we set the integral bounds in \eqref{eq:instant_change_of_variables_density} to: $t_0=0$ and $t_1=1$. Instead of backpropagating through solver steps, we use a memory efficient adjoint method to obtain the gradients \cite{farrell2013automated,chen2018neural}. We use an adaptive ODE solver \cite{dormand1980family}, whose error is bounded, in contrast to fixed-step solvers that have bounded number of steps. In the transformation $f$, we use permutation equivariant layers based on \eqref{eq:deepset} or \eqref{eq:attention}. One important measure of scalability of ODE models is the number of evaluations that the solver performs. In the experiments we show our exact trace model requires fewer steps.

To sample new realizations, we first sample $\SetSize$ points from the base distribution $\ZSet \sim q(\ZSet)$. We then solve an ODE in the opposite direction (from $1$ to $0$) to get samples $\Set$ from the learned distribution $p(\Set)$.

\section{Related work}\label{sec:related}

\subsection{Models for exchangeable data and sets}

\citet{zaheer2017deep,qi2017pointnet} propose permutation equivariant layers (\eqref{eq:deepset}) with universal approximation. \citet{lee2019set} demonstrate the same property holds for the attention based network and propose an extension for large sets that uses a small number of learnable inducing points. \citet{wagstaff2019limitations} study the theoretical limitations of invariant functions with sum aggregation. Alternative aggregation schemes include max-pooling, Janossy pooling \citep{murphy2018janossy}, featurewise sort pooling \citep{zhang2019fspool}, and using graph methods \citep{grover2016node2vec}.

\citet{korshunova2018bruno} treat points independently and use an exchangeable Student-t base distribution. \citet{bender2019permutation} use an autoregressive function on ordered sequences and optimize the following likelihood:
\begin{align}\label{eq:order_likelihood}
    p(\Set) = \frac{1}{\SetSize !} p(\SetElement_{(1)}, \dots, \SetElement_{(\SetSize)}) .
\end{align}
An obvious disadvantage is the requirement to impose a canonical order, which may not be possible and is often arbitrary.  Other autoregressive models for \textit{set-like} data include graph generation on breath-first search ordering \citep{you2018graphrnn}; and ordering on z-axis for meshes and point clouds \citep{nash2020polygen,sun2020pointgrow}. In temporal point processes the order of points is known so existing approaches mostly use autoregressive models \citep{du2016recurrent, mei2017neural,shchur2019intensity,shchur2020fast}.

In contrast, \citet{yang2019pointflow} model point clouds with a variational autoencoder, having one CNF for posterior and another CNF for point locations. \citet{Yuan2020Variational} apply the variational autoencoder approach in spatial point processes by defining a nonparametric kernel. VAE approach relies on de Finetti’s theorem \cite{de1937prevision, o2009exchangeability}:
\begin{align}\label{eq:definetti}
    p(\Set) = \int p(\vz) \prod_{\SetElement_i \in \Set} p(\SetElement_i | \vz) d\vz ,
\end{align}
and maximizes the evidence lower bound:
\begin{align*}
    \log p(\Set) \ge \E_q \left[ \log p(\Set | \vz) \right] - \KL \left[ q(\vz | \Set) || p(\vz | \Set) \right] .
\end{align*}
Using amortized inference, parameters of the posterior are defined as a permutation invariant function of $\Set$. Taking the gradients w.r.t.\ the samples is enabled with the reparametrization trick \citep{kingma2013auto}.
\citet{Yuan2020Variational} demonstrate an application of their approach for recommender systems by embedding the non-spatial data with a graph neural network. Our model supports this too, without any changes to the architecture.

\subsection{Equivariant continuous normalizing flows}

The idea of using CNFs together with equivariant functions has been proposed by \citet{kohler2020equivariant} in the context of multi-body physical systems by modeling the Gibbs distribution with a simple equivariant function, based on a Gaussian kernel. In \Secref{sec:model} we offer ways to use strictly more expressive equivariant functions with closed-form trace. In the Supplementary Material we demonstrate how our method outperforms Gaussian kernel by a large margin. \citet{rezende2019equivariant} propose modeling physics systems by transforming the initial samples with the Hamiltonian dynamics. This method is similar to CNFs in general, but requires variational approximation for training. \citet{li2020exchangeable} propose modeling sets similar to our baseline model with stochastic trace. We demonstrate how this is not scalable and propose an efficient alternative. We also show how this can be applied for spatial point process models with exact likelihood --- an important, previously unsolved problem.

\subsection{Improvements to Neural ODEs}

Neural ODEs \citep{lu2018beyond, chen2018neural} were recently proposed as a continuous equivalent of the widely adopted residual architecture \citep{he2016deep}. By taking advantage of their properties we can construct a continuous version of the normalizing flow. \citet{grathwohl2018ffjord} demonstrate the benefits of this when using a stochastic trace estimator. \citet{chen2019neural} first used general network decoupling to train CNFs which led to improvements in performance. We use their approach to extend our base models with fixed trace. A lot of recent work focused on improving the stability of ODEs, including better representation power \citep{dupont2019augmented}, improved training \citep{gholami2019anode, zhuang2020adaptive}, regularization of the solver dynamics \citep{finlay2020train, kelly2020learning}, etc. Our work is orthogonal to these methods since it focuses on the special case of modeling invariant data but can also be easily combined with them.

\subsection{Point processes}

Cox process \citep{cox1955some} is a doubly stochastic Poisson process that defines intensity as a realization of a random field, e.g.\ a Gaussian process. For example, a Neyman-Scott process \citep{neyman1958statistical} generates initial (parent) points from a Poisson process, and final (children) points around parents. A more specific example is a Mat\'ern cluster process \citep{matern2013spatial} that places children in balls centered around parents. A variation where children follow a normal distribution is called a modified Thomas process \citep{thomas1949generalization}. Even though we do not specify the random field explicitly, the randomness from the base density is allowing us to model these processes with equivariant transformations.

Explicit interactions can be defined with a Gibbs process by assigning energy values to point configurations, originally defined in the context of statistical physics \citep{ruelle1969statistical}. However, this comes at a cost of calculating the normalization constant to find exact densities. Besides pseudolikelihood \citep{besag1975statistical}, other works used logistic composite likelihood \citep{clyde1991logistic} and Monte Carlo methods \citep{huang1999improvements}. We avoid this by designing a model with exact likelihood at no flexibility trade-off.

\section{Experiments}\label{sec:experiments}

In the experiments we aim to show how having an exact trace model does not limit the expressiveness --- on the contrary --- we get the benefits of the faster and more stable training. First, we show that our versions of the models are equivalent to their original implementation, then demonstrate the modeling capacity for the point processes, and finally, show how we can scale to bigger datasets.

The detailed explanation of the data processing, hyperparameter tuning, and additional results can be found in the Supplementary Material.\footnote{\url{https://www.daml.in.tum.de/scalable-nf}} We train all of our models on a single GPU (12GB).

\subsection{Matching performance with the exact trace}

We use established tasks \cite{lee2019set,li2020exchangeable} to compare our novel models with their original versions. By obtaining similar performance, we demonstrate their equivalence while additionally having exact trace computation.

\textbf{Maximum value regression.}
We construct a toy task where the goal is to retrieve the maximum value from a set. We use sets of real numbers, and compare the deep set and attention original architectures with their zero trace and exact trace versions. We run each experiment $5$ times and report the mean and the standard deviation of the test set mean absolute error. Table \ref{tab:traditional_result} shows that there is no significant difference in performance and that our models sometimes perform better.

\textbf{Counting the number of unique digits.}
We construct sets of digits from MNIST dataset where the target is the number of unique classes in a set. The model needs to learn to classify images and then count the number of unique classes. Our end-to-end model processes the digits with a multilayer convolutional neural network, aggregates the embeddings using deep set or attention, and outputs the $\lambda$ parameter of a Poisson distribution. We train by maximizing the log-likelihood and evaluate the test performance with accuracy, defined as the frequency of predicting the Poisson mode that matches the true value, as in \cite{lee2019set}.

We run each experiment $5$ times and report the mean and the standard deviation of the test set accuracy. The results in Table \ref{tab:traditional_result} again indicate that our architecture matches the performance of the original models. Note that we set the latent dimension $d_g$ and the hidden dimensions such that the total number of parameters in the exact trace models and the original models is approximately the same.

\begin{table}
    \centering
    \begin{tabular}{lcc}
 &          Deep set  & Attention \\
\multicolumn{2}{l}{Max value regression} & \\
\midrule
Unchanged   & 0.032 $\pm$ 0.003 & 0.029 $\pm$ 0.002 \\
Zero trace  & 0.032 $\pm$ 0.003 & \textbf{0.022 $\pm$ 0.002} \\
Exact trace & \textbf{0.027 $\pm$ 0.003} & 0.036 $\pm$ 0.011 \\
&  &  \\
Counting digits & & \\
\midrule
Unchanged   & 0.838 $\pm$ 0.010 & 0.854 $\pm$ 0.012 \\
Zero trace  & 0.836 $\pm$ 0.017 & \textbf{0.872 $\pm$ 0.017} \\
Exact trace & \textbf{0.843 $\pm$ 0.013} & 0.842 $\pm$ 0.005 \\
&  &  \\
Point cloud & & \\
\midrule
Unchanged   & \textbf{0.7685 $\pm$ 0.017} & 0.7495 $\pm$ 0.007 \\
Exact trace & 0.7523 $\pm$ 0.008 & \textbf{0.7714} $\pm$ 0.012 \\
\end{tabular}

    \caption{Mean absolute error on maximum value regression (top), accuracy of counting the  unique digits in a set (middle), and accuracy of point cloud classification (bottom).}
    \label{tab:traditional_result}
    \vspace{-0.5cm}
\end{table}

\textbf{Point cloud classification.} Similar to \cite{li2020exchangeable}, we use labeled point clouds \cite{wu20153d} in a supervised classification task. We feed in a set of $512$ points to a two-layer convolutional network followed by an ODE with the deep set or attention network, apply max-pooling over all the elements and output the logits with a two-layer feedforward network. The hidden dimensions range in $(32, 128)$, and there are $40$ classes in total. The results in Table \ref{tab:traditional_result} show that both the exact trace and original transformation perform similarly.

Note that the models in these experiments were not tuned to achieve the best performance on these tasks but rather to compare the performance between two different implementations of the deep set and attention architecture, the original versus the exact trace version. We can conclude that there is no significant difference in these two implementations, as intended.
x
\subsection{Modeling point processes}

\textbf{Synthetic data.} For all datasets we simulate $1000$ realizations on $(0, 1)\times(0, 1)$ area. \underline{Thomas} dataset is simulated by first sampling $m \sim \Poisson(3)$ parents uniformly, and then for each, we sample $n_i \sim \Poisson(5)$ children from the normal distribution centered on the parents ($\sigma = 0.01$). In \underline{Mat\'ern} the children are uniformly sampled inside discs centered on parents ($R = 0.1$). \underline{Mixture} is an inhomogeneous process that generates $n \sim \Poisson(64)$ points from a mixture of $3$ normal distributions.

\textbf{Real-world data.} \underline{Check-ins NY} \cite{cho2011friendship} is a collection of locations of social network users. We consider points in New York that have $1938$ unique users. The single realization corresponds to all the recorded locations for one user. We also construct a smaller dataset for a different city (\underline{Check-ins Paris}) that has $286$ users. \underline{Crimes}\footnote{\url{https://nij.ojp.gov/funding/real-time-crime-forecasting-challenge}} dataset contains daily records of locations and types of crimes that occurred in Portland. Each day contains between $298$ and $736$ points with $480$ on average.

\textbf{Baselines.} Inhomogeneous Poisson process (\underline{IHP}) models $p(\Set) = \prod_i p(\SetElement_i)$, i.e., it assumes independence between the points. We use a normalizing flow with spline coupling layers \cite{durkan2019neural}. IHP cannot capture interactions between the points but models a fixed density on $\BoundedRegion$. Additionally, we use an \underline{\smash{autoregressive}} model \cite{bender2019permutation}, where the input points are sorted on the first dimension and we maximize \eqref{eq:order_likelihood}. We also include the model without exact likelihood, an importance weighted autoencoder (\underline{IWAE}). For the permutation invariant encoder we either use deep set or attention layers with layer normalization \cite{ba2016layer} and max-pooling to output parameters of the posterior $p(\vz | \Set)$. The likelihood term $p(\Set|\vz)$ is an IHP conditioned on $\vz$. To estimate the likelihood (\eqref{eq:definetti}) on the held-out data we use $5000$ samples \cite{burda2015importance}. The detailed description of data processing steps and the hyperparameters we used is in the Supplementary Material.

\begin{figure}
    \centering
    \includegraphics{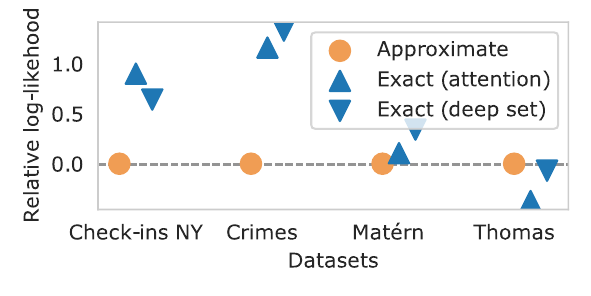}
    \vspace*{-0.3cm}
    \caption{Relative difference in log-likelihood (higher is better) centered at the approximate trace model scores. We compare two different architectures for our approach.}
    \label{fig:relative_nll}
\end{figure}

\begin{table*}
    \centering
    \small
    \begin{tabular}{lcccccc}
{} &                        Check-ins NY   &  Check-ins Paris &           Crimes &         Mat\'ern &          Mixture &           Thomas \\
\midrule
IHP             & -1.61 $\pm$ 0.02 (0.035) & -2.87 $\pm$ 0.04 (0.026) & -2.34 $\pm$ 0.01 (0.002) & -0.21 $\pm$ 0.00 & -2.06 $\pm$ 0.00 & -0.01 $\pm$ 0.00 \\
Autoregr.       & -1.69 $\pm$ 0.02 (0.016) & -3.39 $\pm$ 0.05 (0.022) & -1.61 $\pm$ 0.06 (0.006) & -0.38 $\pm$ 0.02 & -1.55 $\pm$ 0.01 &  0.13 $\pm$ 0.02 \\
IWAE            & -2.03 $\pm$ 0.07 (0.035) & -3.57 $\pm$ 0.02 (0.033) & \textbf{-2.35 $\pm$ 0.00} (0.002) & -0.59 $\pm$ 0.01 & -2.05 $\pm$ 0.01 & -0.23 $\pm$ 0.00 \\
CNF (Our)       & \textbf{-2.30 $\pm$ 0.09} (\textbf{0.012}) & \textbf{-3.58 $\pm$ 0.03} (\textbf{0.014}) & -2.34 $\pm$ 0.01 (0.002) & \textbf{-0.77 $\pm$ 0.08} & \textbf{-2.07 $\pm$ 0.00} & \textbf{-0.55 $\pm$ 0.00} \\
\end{tabular}

    \caption{Per-point negative log-likelihood (in brackets: Wasserstein distance of the between points distance distributions) on point process data. Lower values are better for both metrics.}
    \label{tab:pp_nll_results}
\end{table*}

\begin{figure*}
    \centering
    \includegraphics[height=2.42cm]{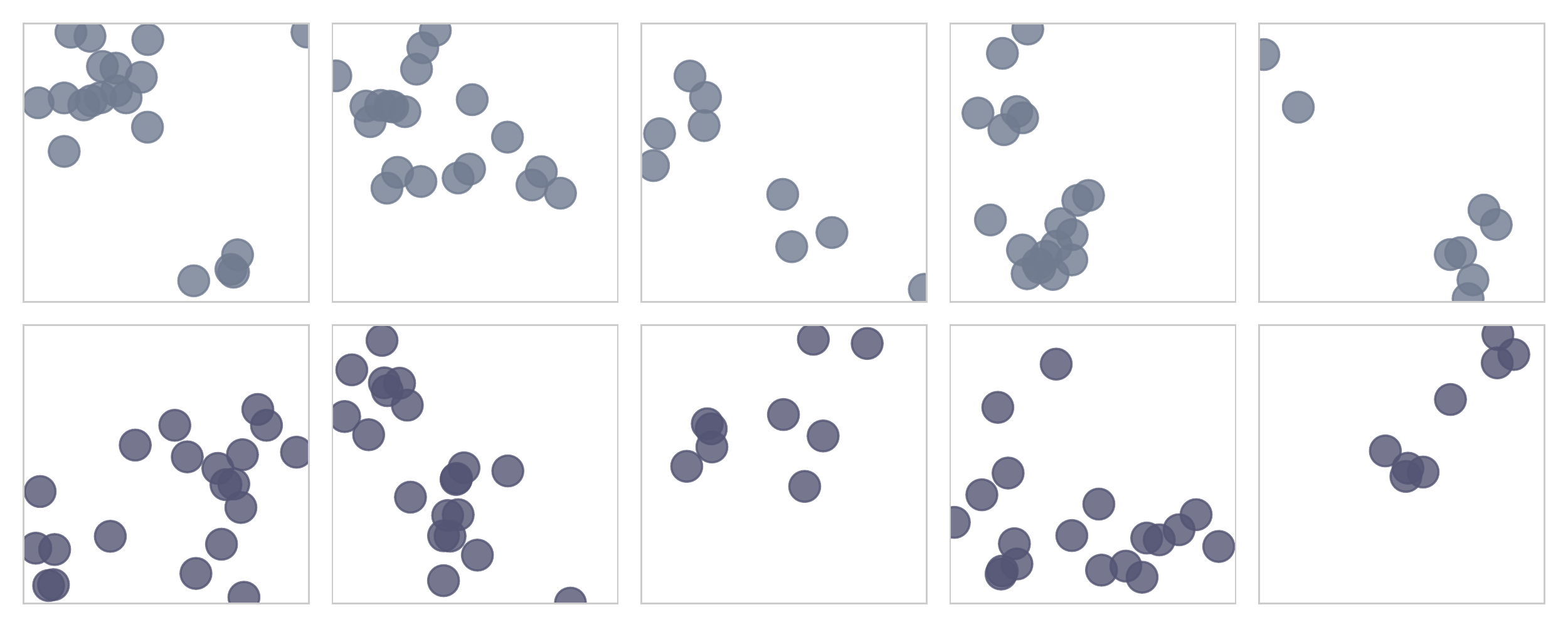}
    \includegraphics[height=2.42cm]{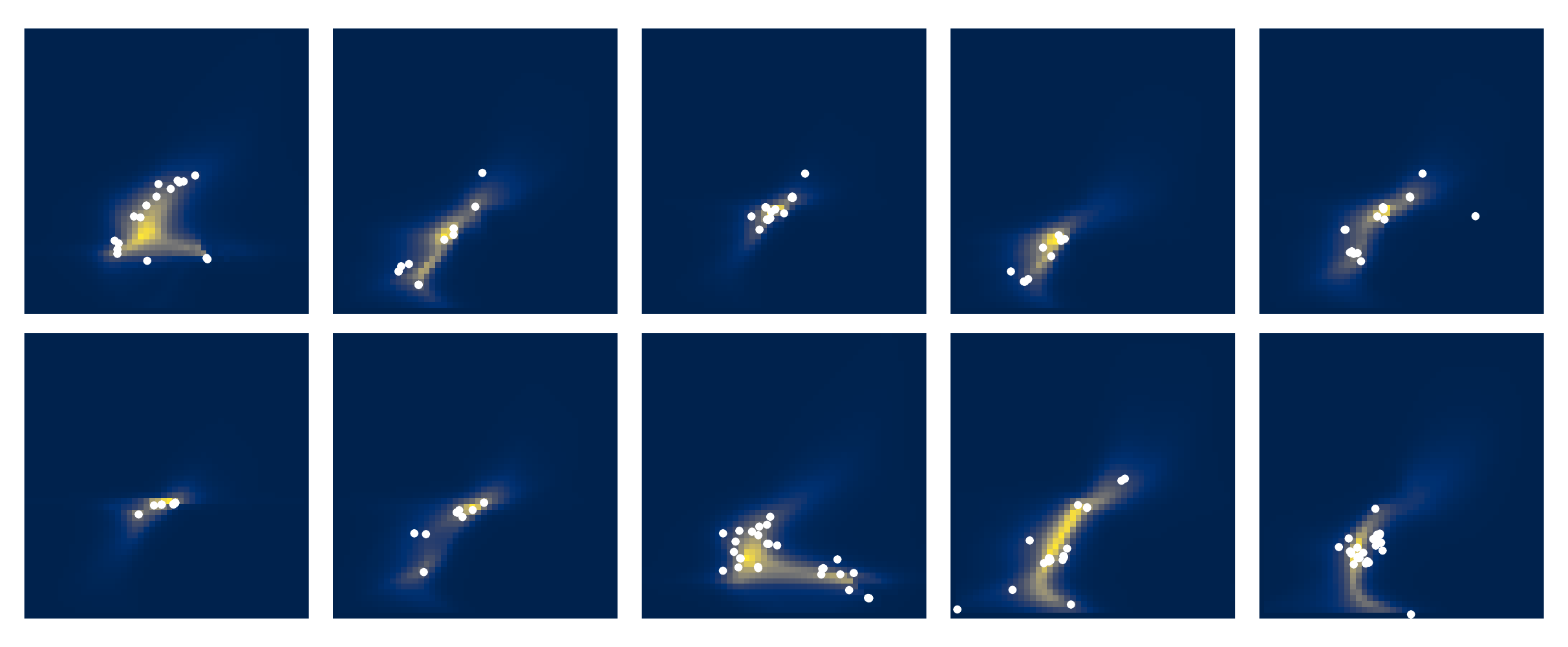}
    \includegraphics[height=2.46cm]{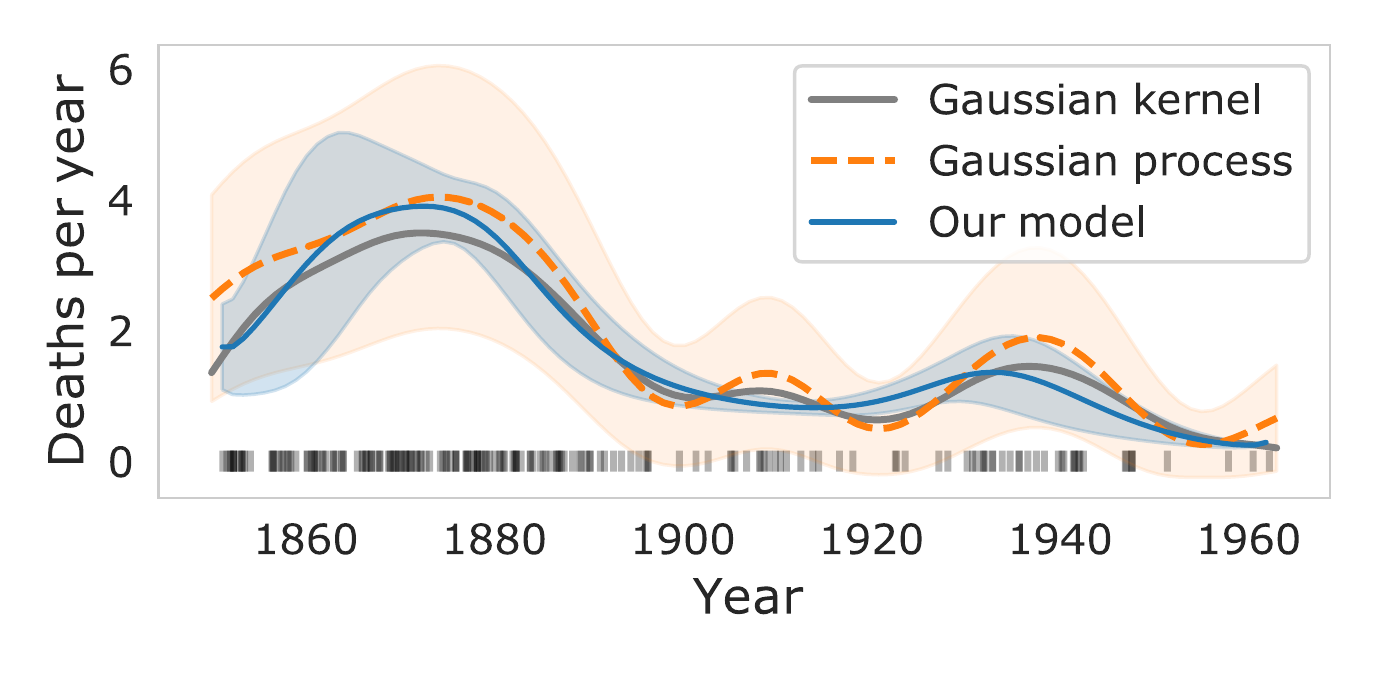}
    \vspace*{-0.5cm}
    \caption{Left: Thomas data (top) and samples (bottom). Middle: Density realizations on NY data. Right: Comparison to GP-PP model.}
    \label{fig:samples}
\end{figure*}

\textbf{Results.} Datasets are split into training, validation and test sets (60\%-20\%-20\%). We train with early stopping, use mini-batches of size $64$ and Adam optimizer with the learning rate of $10^{-3}$ \citep{kingma2014adam}. The loss we use is per-point negative log-likelihood, i.e.\ $\mathcal{L}(\Set) = -\log p(\Set) / \SetSize$. Table \ref{tab:pp_nll_results} reports the mean and the variance of the loss on a held-out test set averaged over $5$ runs for models selected based on their validation set performance. Our method beats others in most of the cases, and is within the margin of error for the remaining. In \Figref{fig:relative_nll}, we compare our method to an approximate trace approach. We see that the relative difference to the log-likelihood of an approximate model is mostly in favour of our method.

We additionally quantitatively test the sampling quality by comparing the distributions of in-between point distances for real data and new model samples. Again, our model outperforms the competitors. We conclude that CNFs are well suited for modeling exchangeable data. Further, samples in \Figref{fig:samples} (Left) qualitatively demonstrate that we capture the true underlying process, in this case, the clustering behavior.
We can see that the individual points group around, forming small clusters, as expected. We additionally compute the Ripley's K function that measures the homogeneity of the process, where the higher number denotes a more clustered process. The real data has $K=1.92$ and our model achieves $K=1.93$, almost a perfect match. For reference, the baseline inhomogeneous model has $K=1.41$.

Finally, in the Supplementary Material we show how our method can be orders of magnitude faster during test time evaluation. This is because approximate trace models have to compute the trace exactly in this setting, thus, even if the training steps take a small amount of time, the model in the production can be prohibitively slow. We also demonstrate that removing variance imposed by stochastic estimators can lead to much better convergence in the early stages of training. This means our method achieves better results faster, and the results stay better according to \Figref{fig:relative_nll}.

\begin{figure}[h!]
    \centering
    \begin{tikzpicture}
        \node[inner sep=0pt] at (-3.5, 1) {\small $\SetSize = 20$};
        \node[inner sep=0pt] at (-3.5, 0) {\small $\SetSize = 30$};
        \node[inner sep=0pt] at (-3.5, -1) {\small $\SetSize = 40$};
        \node[inner sep=0pt] at (0, 0) {\includegraphics[width=0.7\linewidth]{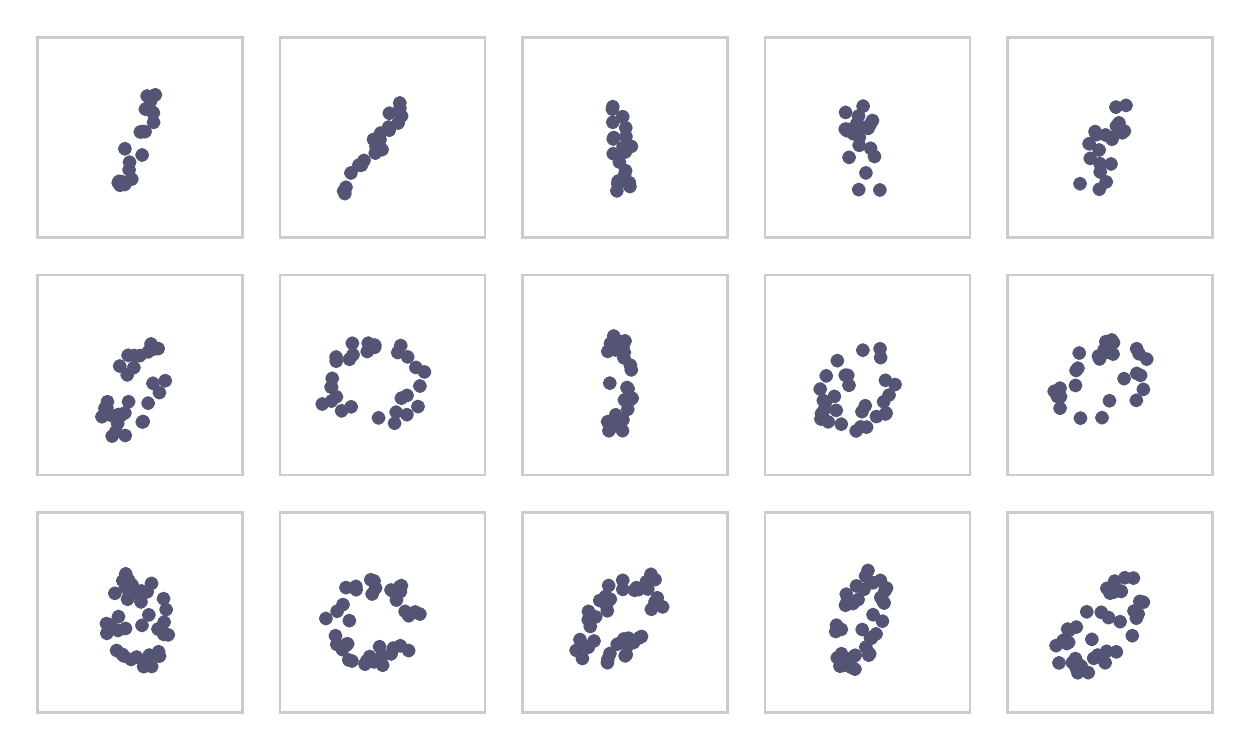}};
    \end{tikzpicture}
    \caption{Setting different set sizes $\SetSize$ results in different types of samples for data containing discretized digits.}
    \label{fig:cardinality}
\end{figure}
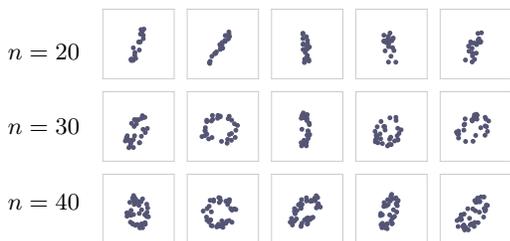

\textbf{Set cardinality.} Above, we followed previous works and used the loss that ignores the cardinality of the set captured in the $p(\SetSize)$ term. The empirical distribution of the number of points in our data is almost always simple and unimodal, meaning that the burden of the modeling is on the locations of the points and their interactions. It is easy to have a full point process model by adding $p(\SetSize) = \Poisson(\Intensity)$, where $\Intensity$ is learnable.

To show this we construct a toy dataset consisting of points that form shapes of two digits: digit $0$ that has $\SetSize=40$ points and digit $1$ with $\SetSize=20$. Learning $p(\SetSize)$ is trivial in this case, the hard part is learning the distribution of point locations given $\SetSize$, $p(\Set | \SetSize)$. After fitting the data with our attention-based model, we show in \Figref{fig:cardinality} how sampling with different $\SetSize$ gives different shapes, as expected.

\textbf{Stochastic density.}
Since the density $p(\Set)$ is influenced by interactions between the points, it is not stationary but depends on $\Set$ (unlike inhomogeneous model). We can draw a new \textit{density realization} on $\BoundedRegion$ by sampling $\Set \sim p(\Set)$. To be precise, we first draw $\Set \sim p(\Set)$, which defines the conditional density at non-observed points $\SetElement \in \BoundedRegion, \SetElement \notin \Set$. We can sample from this distribution or use it to detect the out of distribution data. \Figref{fig:samples} (Middle) shows different conditional densities based on the real data samples for a model trained on Check-ins NY data. Having different density realizations connects to the notion of the doubly stochastic processes \cite{cox1955some}.

A popular way to define a Cox process is by defining the random field with a Gaussian process \citep{rasmussen2016gp}, and then drawing an intensity function as a random realization. \Figref{fig:samples} (Right) shows how a Gaussian process modulated Poisson process (GP-PP) \citep{lloyd2015variational} fits the coal mining disaster data, with $191$ points collected from 1851 to 1962. We additionally plot the empirical intensity measure estimated with a Gaussian kernel. Finally, we plot the mean and the standard deviation of $100$ \textit{intensity function} realizations from our model. It closely matches the underlying intensity and provides the uncertainty measure. We get the intensity function from our model by scaling the density function with the learned parameter $\Intensity$.

The above two examples demonstrate how the stochasticity of the base distribution samples leads to different final conditional distributions, without Gaussian process modulation.

\begin{figure*}
    \centering
    \includegraphics[height=3cm]{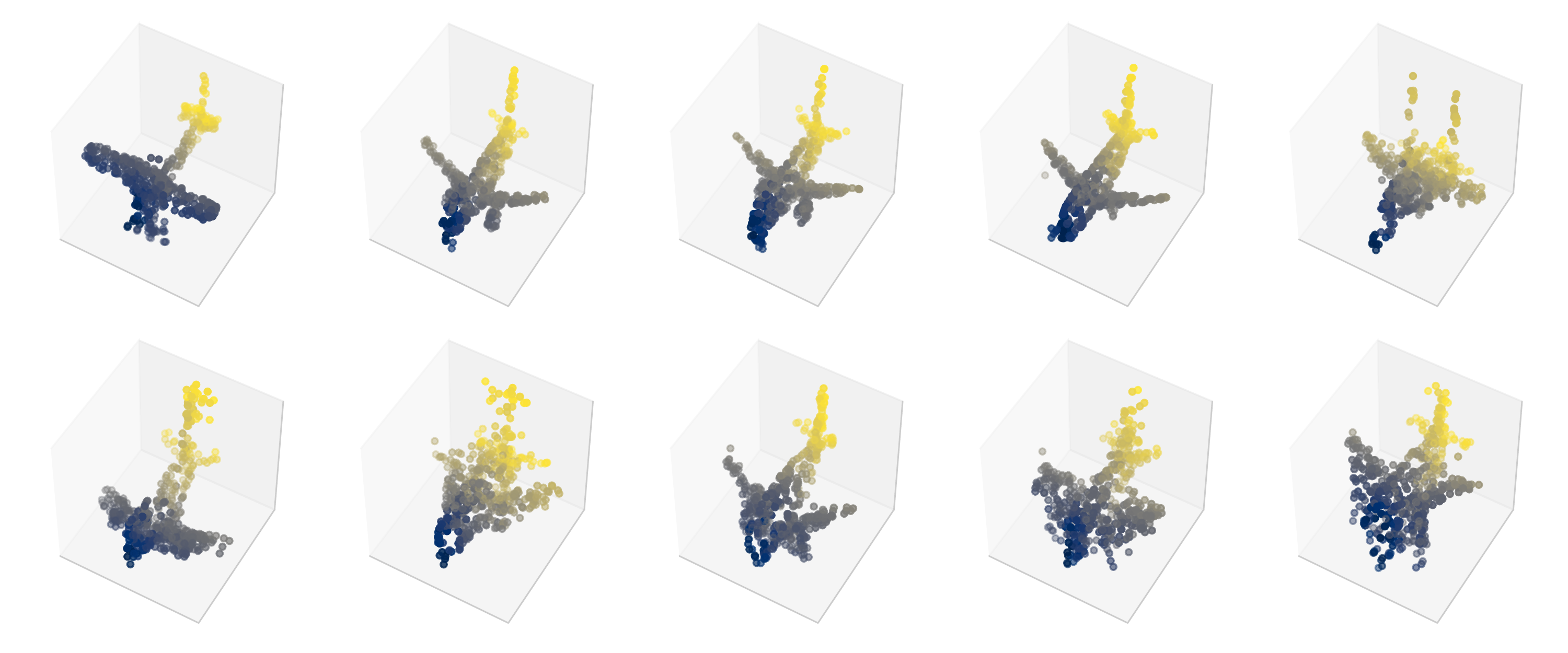}
    \includegraphics[height=3cm]{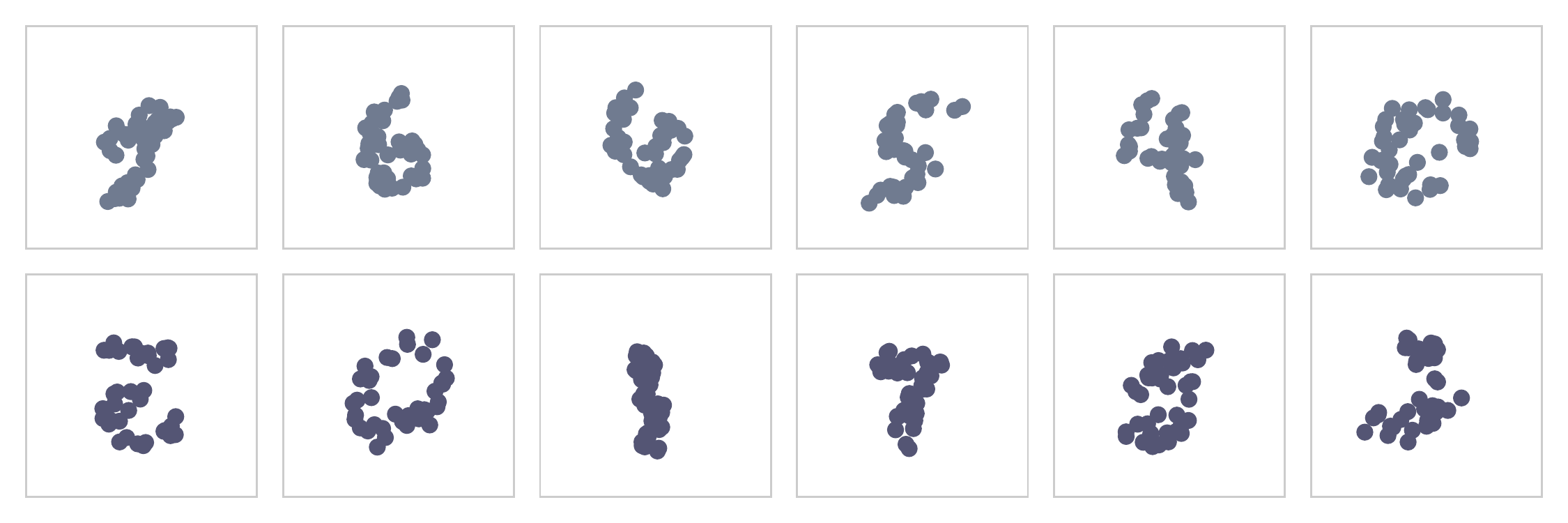}
    \caption{Real data (top), and our model samples (bottom) for airplane class in point clouds data and discretized digits from MNIST.}
    \label{fig:airplane_samples}
\end{figure*}

\begin{figure}[h]
    \centering
    \includegraphics[height=3.5cm]{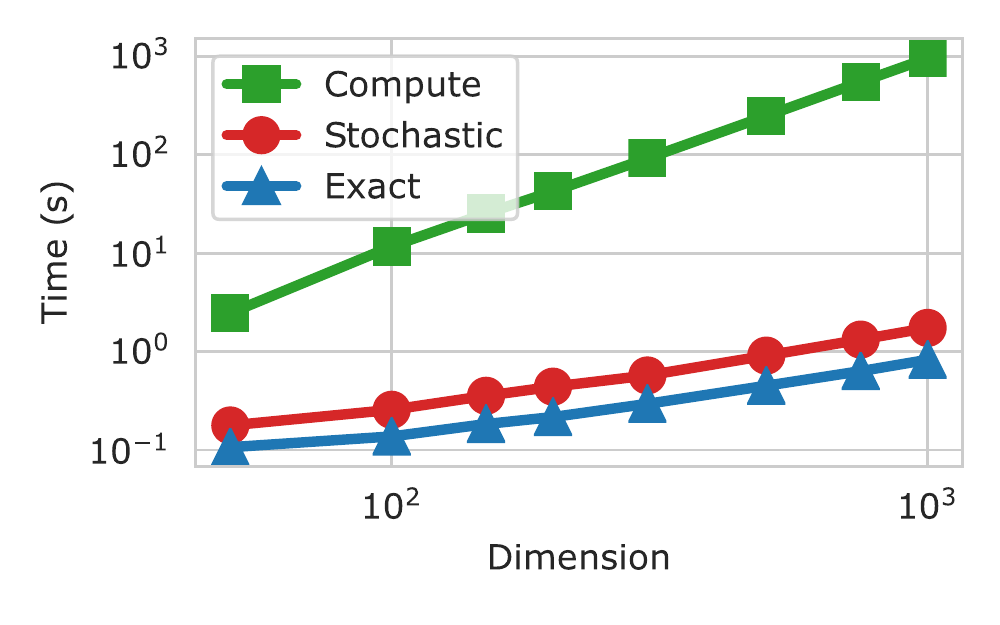}
    \vspace*{-0.5cm}
    \caption{Comparing the time (log scale) to evaluate trace for different estimators across different dimension sizes of data.}
    \label{fig:trace_speed}
\end{figure}

\subsection{Scaling to large sets}

In the previous section we demonstrated that CNFs are a good choice for modeling exchangeable data. Now we test how different trace calculation approaches affect the performance on bigger datasets.

First, we compare the speed of different trace estimation methods, namely: our exact trace model, stochastic trace estimation with Hutchinson's estimator, and computing the trace exactly without using any specialized architectures. We use the same type of the model and the same number of parameters. We record the time to calculate the trace on dimension sizes in the range $(50, 1000)$. \Figref{fig:trace_speed} shows that our method is faster than both competing approaches. Finding the trace exactly on arbitrary architectures (not using methods from Section~\ref{sec:model}) becomes untractable very soon. That means that even if we used relatively fast stochastic trace estimation during training, our model would be of limited use in production when we want to know the exact density.

\begin{figure}
    \centering
    \includegraphics{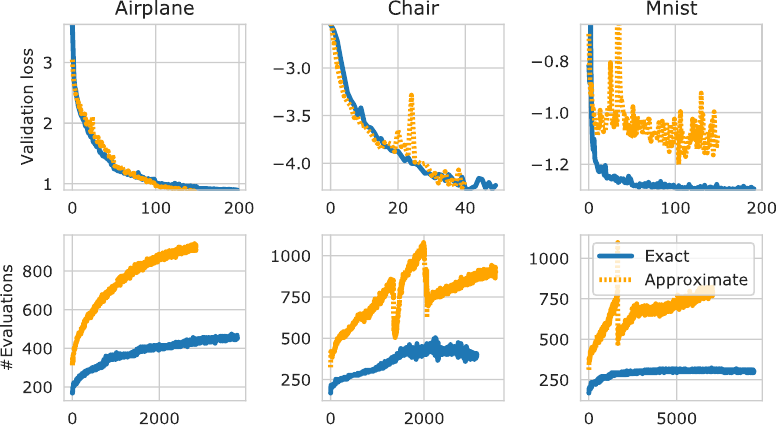}
    \vspace*{-0.3cm}
    \caption{Comparison of the number of solver evaluations and the validation loss during training for exact and approximate trace.}
    \label{fig:evaluations_comparison}
\end{figure}

Further, we test the performance on real-world large sets of points. For that we model point clouds, in particular, we use airplane and chair classes \cite{wu20153d} where each sample contains $512$ points. We additionally use discretized MNIST digits dataset, i.e., sample the coordinates of $40$ black pixels and add a small noise to them to avoid degenerate solutions (placing infinite mass on a single point). The model consists of $12$ stacked continuous normalizing flows with the deep set or the attention transformations, and a similar hyperparameter setup as in \citet{li2020exchangeable}. We observe that both versions of the model perform similarly.

\Figref{fig:airplane_samples} shows samples of airplane point clouds and discretized digits from our model. We can see that the model captures the complex data well, and the samples are visually similar to previous works \citep{li2020exchangeable}. Next, we test an important metric that shows scalability --- number of ODE solver evaluations. In \Figref{fig:evaluations_comparison} we can see that having an exact trace in a model allows us to obtain the lower number of total evaluations, meaning that both the training and inference will be faster compared to the stochastic trace estimation. At the same time we do not sacrifice performance, as is seen in the evolution of validation loss over training, where our method matches or outperforms the competing approach. Note that we also observe noisier training and validation curves for approximate trace meaning the training is less stable. An interesting observation is that both models adapt the complexity by requiring more evaluations later in the training, when they learn the underlying distribution better.

\begin{figure}[h]
    \centering
    \includegraphics[height=3.4cm]{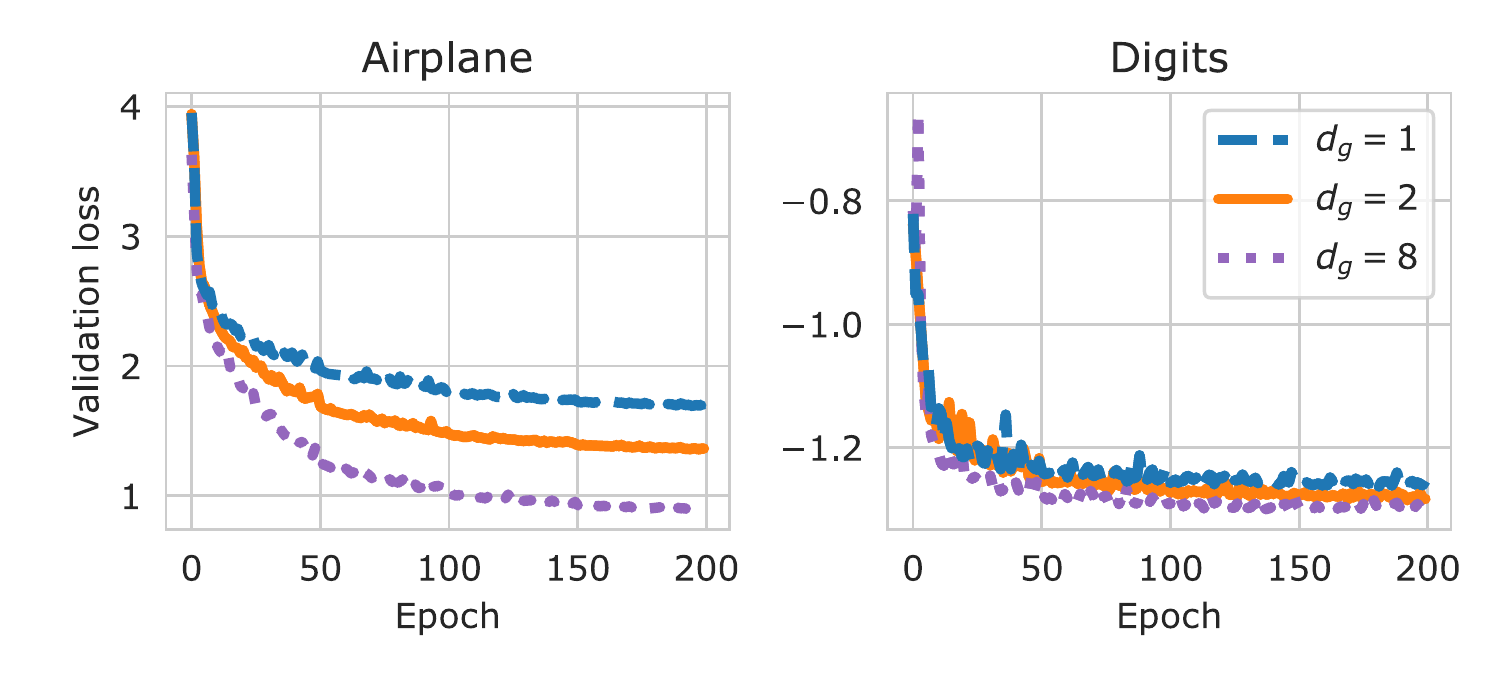}
    \caption{How different latent size $d_g$ influences expressiveness.}
    \label{fig:latent_dim_loss}
\end{figure}

\subsection{Ablation studies}

\textbf{Latent dimension $d_g$.} We do an ablation study for different latent dimensions $d_g$ for the network $g$, as described in \Secref{sec:model}. We test the sizes of $1, 2$ and $8$. \Figref{fig:latent_dim_loss} shows that higher $d_g$ leads to better performance as expected, but in the case of the digits dataset, we reach the point of diminishing returns very fast.

\textbf{Volume preserving flow.} The version of our model which has the constant zero trace only accounts for the interactions between the points, but not the full transformation of the points themselves. This means that if the base process is homogeneous, the volume preserving flow cannot learn the inhomogeneous part of the process, i.e., the fixed density over $\BoundedRegion$. In the Supplementary Material we show how it performs against IHP model and demonstrate that on some datasets it is enough to use the model with zero trace transformations alone. In the end, it is the best to use the full model with exact trace since it captures both the inhomogeneous part and the interaction between the points.

\textbf{ODE regularization.} The goal of our method is to reduce the variance during training, which leads to having a better performance both during training and inference. Most of the recent works on improving neural ODE models (see \Secref{sec:related}) focused on learning functions $f$ that reduce the \textit{stiffness}, meaning we get better behavior in the ODE solver. In the point process experiments we tried Frobenius norm regularization \cite{finlay2020train}, but did not observe significant changes in performance.

\textbf{Induced attention.} \citet{lee2019set} propose an alternative to the attention layer that uses a set of inducing points $m$, aiming to reduce the computation cost when $m \ll \SetSize$. We compare it with the usual self attention on the point process data, and notice that it usually performs worse. The details can be found in the Supplementary Material. Having an extra hyperparameter is another drawback of this method. However, the lower memory footprint of induced attention can help train models on sets with very large cardinality.

\section{Conclusion}\label{sec:conclusion}

In this work we presented a novel model for point processes and random sets that uses continuous normalizing flows with equivariant functions that has the exact trace of the Jacobian available. The exact trace, in particular, allows us to scale to bigger datasets since this approach is more efficient but at the same time, it achieves better performance.

Throughout the experimental evaluation we demonstrated how we can successfully model complex point processes with interactions between the points while having tractable likelihood and straightforward sampling. This is the first spatial point process model with these properties. We further show that our model outperforms all other competitors on this task. At the same time, it does not require any special assumptions on the data, like independence between the points, nor does it need canonical ordering.

Our method relies in particular on two types of deep learning models: continuous normalizing flows and equivariant neural networks. By utilizing new techniques developed in these two subfields (new types of equivariant layers and ODE regularizations) we hope it will perform even better in the future.

\section*{Acknowledgments}

This research was supported by BMW AG and by the German Federal Ministry of Education and Research (BMBF), grant no. 01IS18036B. The authors of this work take full responsibilities for its content.

\bibliography{references}
\bibliographystyle{icml2021}

\clearpage
\appendix
\section{Theoretical background}

\begin{theorem}[Equivariant flows]\label{th:equivariant-flow}(Adapted from \citet{papamakarios2019normalizing}, Section 5.6, Lemma 1) Let $p(\Set)$ be the density function of a flow-based model with transformation $f:\R^{\SetSize \times \Dim} \rightarrow \R^{\SetSize \times \Dim}$ and base density $q(\ZSet)$. If $f$ is equivariant with respect to $\Gamma$ and $q$ is invariant with respect to $\Gamma$, then $p(\Set)$ is invariant with respect to $\Gamma$.
\end{theorem}
\begin{proof}
    See \citet{papamakarios2019normalizing}.
\end{proof}

Theorem \ref{th:equivariant-flow} is satisfied for permutation group $\Gamma$ if we, e.g., use normal or uniform base distribution, and an equivariant transformation.

\begin{theorem}[Equivariant ordinary differential equation]\label{th:equivariant-ode}
    Let $f(\Set)$, $f:\R^{\SetSize \times \Dim} \rightarrow \R^{\SetSize \times \Dim}$ be the solution of the ordinary differential equation $d\ZSet(t)/dt = g(\ZSet(t), t)$ on $I = [t_0, t_1]$ with an initial condition $\ZSet(t_0) = \Set$. If $g$ is equivariant with respect to $\Gamma$ then so is $f$.
\end{theorem}
\begin{proof}
    $ \gamma f(\Set) = \gamma f(\ZSet(t_0)) = \gamma \ZSet(t_0) + \gamma \int_{t_0}^{t_1} g(\ZSet(t), t) dt = \gamma \ZSet(t_0) + \int_{t_0}^{t_1} g(\gamma \ZSet(t), t) dt = f(\gamma \ZSet(t_0)) = f(\gamma \Set), \gamma \in \Gamma$.
\end{proof}

Theorem \ref{th:equivariant-ode} shows that the equivariant drift in a continuous normalizing flow gives an invariant density. In our work, we use permutation equivariant functions to get permutation invariant densities.

\section{Dataset preprocessing and generation}

\textbf{Matching performance data.}
In \Secref{sec:experiments} we wanted to see if our exact trace models can match the performance of the original architectures. First part of the experiments included using a synthetic dataset. We generated sets of $n = 10$ numbers, uniformly sampled between $(a, b)$, where $a$ and $b$ are uniformly sampled from $(0, 1)$. Therefore, the random interval from which we get the numbers is always different. The goal is to return the biggest value in the set. The model needs to learn the identity mapping, followed by the max-pooling operation. We adapted the task from \citet{lee2019set}.

The second task is counting the number of unique digits in a set, also from \citet{lee2019set}. Each input contains $n = 10$ images, normalized from MNIST dataset. The count is defined as the number of unique classes in the set. For example, if there are $5$ images that represent digit $0$ and $5$ that represent $1$, then the count is equal to $2$.

Finally, we use labeled point clouds \citep{wu20153d} for classification task. A single input is a set of $512$ point and a corresponding class. There are $40$ unique classes. We center and normalize the data such that each example has zero mean, and unit variance.

MNIST and two classes from point cloud dataset (airplanes and chairs) are used for generation experiments as well. We do the same preprocessing steps.

\textbf{Point process data.}
We generate synthetic datasets on $\BoundedRegion = (0, 1) \times (0, 1)$ region and take care of the edge effects by simulating on a region larger than $\BoundedRegion$, e.g. for Mat\'ern clustering process with $R = 0.1$ we extend the simulation square by $0.1$ on all sides. The points that fall out of the observed region are discarded. Thomas and Mat\'ern are simulated as it is described in the main text. The mixture dataset uses 3 normal distributions with means $(0.3, 0.3)$, $(0.5, 0.7)$, $(0.7, 0.3)$ and a diagonal covariance $0.05$.

Checkins-NY and Checkins-Paris use data of user location logs that are recorded over the period of 2009 and 2010. All locations for a user are a single realization of the underlying process. Initial data contains just under 6.5 million locations.
We take the locations from two cities, New York, between coordinates $(40.6829, -74.0479)$ and $(40.8820, -73.9067)$, and Paris, between $(48.5, 2)$ and $(49, 3)$, respectively. Data contains very large sets, largest one having $904$ elements. \Figref{fig:checkins-data} illustrates these two datasets.

Crimes dataset contains recordings of crimes that happened in the city of Portland from March 01 2012 to December 31 2012, with $309$ days in total. It contains $146927$ logs, each with a time, coordinate and type of case. We aggregate all the occurrences in one day to get a single realization. This gives us very large sets, smallest one having $298$ and largest $736$ elements.

\begin{figure*}
    \centering
    \includegraphics{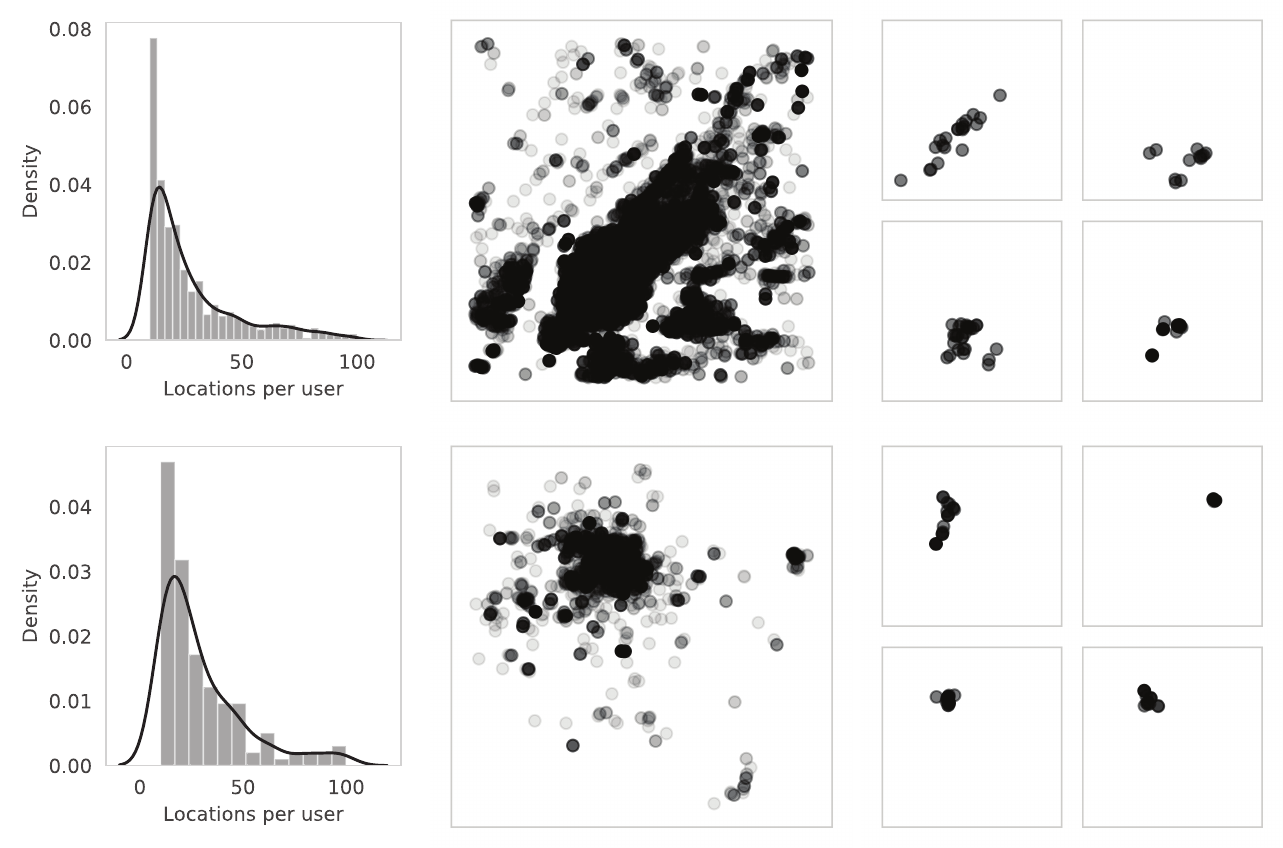}
    \caption{Upper row: Checkins-NY. Bottom: Checkins-Paris. From left to right: distribution of the number of locations per user after
    627 preprocessing; all the locations for all the users; sets of locations for four randomly sampled users.}
    \vspace*{-0.2cm}
    \label{fig:checkins-data}
\end{figure*}

\section{Implementation}

To implement CNFs we extend the library provided by \citet{chen2018neural}\footnote{\url{https://github.com/rtqichen/torchdiffeq}} to allow permutation invariant densities.

\subsection{Normalizing flow layers}

In the following we review the traditional normalizing flows, i.e., different ways to define $f : \R^d \rightarrow \R^d$ with the tractable inverse and the determinant of the Jacobian.
We denote random variables $\vz,\vx \in \R^d$ coming from the base and target density $q$ and $p$, respectively. The forward map $\vz \mapsto \vx$ is parametrized with $f$, and the inverse map $x \mapsto z$ with $f^{-1}$. In all of the models, the analytical inverse exists. In practice we often focus on parametrizing the inverse map, allowing efficient density estimation for maximum likelihood training.

\textbf{Coupling layer}  \citep{dinh2014nice,dinh2016density} defines the transformation of one part of the data point $\vx$ conditioned on the rest. If $\vx \in \R$, we take first $k$ dimensions and copy them, $\vx_{1:k} = \vz_{1:k}$. The rest go through an affine transformation with coefficients as a function of $\vz_{1:k}$:
\begin{align}\label{eq:affine_coupling}
    \vx_{k+1:d} = \vx_{k+1:d} \odot \exp(s(\vx_{1:k})) + t(\vx_{1:k}),
\end{align}
where $s, t : \R^{k} \rightarrow \R^{d-k}$ are unrestricted neural networks. Notice the Jacobian has a special form where the only non-zero elements are on the diagonal and in one block under the diagonal. The determinant is then the product of the diagonal entries. The inverse is obtained by noticing that $\vz_{1:k}$ is equal to $\vx_{1:k}$, and then using this to obtain $\vz_{k+1:d} = (\vz_{k+1:d} - t(\vx_{1:k})) \odot \exp(-s(\vx_{1:k}))$. The computation complexity is the same for both the sampling and the density estimation.

The function described in \Eqref{eq:affine_coupling} is called an affine transformation. Instead, we can use more expressive functions such as rational-quadratic splines. Details of implementation can be found in \citet{durkan2019neural}\footnote{\url{https://github.com/bayesiains/nsf}}. It defines a monotonic function with $K$ bins where each bin is a rational-quadratic function. Increasing $K$ increases the expressiveness. The inverse and determinant are easy to obtain. In the context of coupling layers, one part of the data defines the spline parameters that transform the rest.

\textbf{Autoregressive layer} transforms the ith dimension based on all the previous values $\vx_{:i}$. We can implement this by processing a sequence $(\vz_1, \dots, \vz_d)$ with a neural network, in our case with a recurrent neural network that outputs the parameters of an affine transformation similar to the coupling layer. The Jacobian is now a lower triangular matrix. When calculating the forward direction, we know all the values of $\vz$, and can use efficient parallel implementations making the calculation fast. However, when inverting element $\vx_i$ we need to know all of the previously inverted values $\vz_{:i}$. This means that autoregressive layer is inherently slow and sequential in one direction. Therefore, parametrizing $\vx \mapsto \vz$ direction is preferred for maximum likelihood training \citep{papamakarios2017masked}, whereas for fast sampling, we would parametrize the $\vz \mapsto \vx$ direction \citep{kingma2016improved}.

\textbf{Combining layers} can be implemented as a composition of functions $f_1 \circ \dots \circ f_k$. In practice, for each transformation we know $f$ and $f^{-1}$ together with the log-determinant of the Jacobian. Therefore, for input $\vz_0$ we calculate $\vz_1 = f_1(\vz_0),\dots,\vz_k =f_k(\vz_{k-1})$. We accumulate the log-determinant at every step and calculate $p(\vz_k)$ with the change of variables formula. The same principle applies for the inverse direction.

\textbf{Why don’t we use CNFs in all models?} We do not use the continuous normalizing flows for baseline models for two reasons. First, some models represent existing works in literature that we reimplement here \citep[e.g.,][]{bender2019permutation}. Second, CNFs are not necessarily always better in terms of efficiency and final performance given the task. Both the variational autoencoder and the inhomogeneous process are defining the density on $\BoundedRegion$ without any constraints. Since invariance is implemented via independence, CNFs are here as good as any other normalizing flow method. On the other hand, if we actually want to model interactions, like we do with our model, we need to have a model that supports full Jacobian and CNFs are a perfect candidate.

\subsection{Training in batches}

All of the described functions work with inputs of shape $(n, d)$. We train in mini-batches, i.e. process multiple sets at a time. Since $n$ varies between them, we pad them all to the biggest set with zeros and keep track of the original size. This is important when calculating the loss since we use per-point negative log-likelihood. Also, when implementing interactions between the elements, we want to remove those that represent the masking. For example, in the attention network, the elements interact with each other through the similarity matrix. If most of our inputs are zero padded, this will give unexpected results after calculating softmax. Therefore, we zero-out the elements corresponding to the padding using an infinity-mask. Same is done for other methods. In case we do not do this before the aggregation, the model will learn slower, get worse results, and in extreme cases not converge at all.

\subsection{Hyperparameter search}

In all of the models we tried placing the weight decay of $10^{-4}$ or $10^{-3}$ on the weights. This sometimes leads to more stable training and better performance. Using smaller learning rate ($10^{-4}$) makes training slow without other benefits. We use learning rate scheduler that halves the learning rate every $50$ epochs. In IHP we use a normalizing flow with $L \in \{1, 5\}$ spline coupling layers, each spline defined with $K \in \{5, 10\}$ knots. Autoregressive model stacks $L_a \in \{5, 10\}$ autoregressive and $L_c \in \{5, 10\}$ set coupling layers. When using attention in IWAE, besides the number of layers $L \in \{1, 5\}$, we define the number of heads $H \in \{1, 8\}$. In all of the above, having more layers (bigger model) does not improve the results further.

\subsection{Other models}

\textbf{Counting unique digits.}
We use a two layer convolutional neural network (filter sizes: 32 or 64; kernel size: 3), with dropout between the layers (0.25 or 0.5 remove probability), and a fully connected layer that outputs a fixed 128 vector representing a single image. All images are combined with either deep set or attention architecture, with hidden dimension sizes of 128. The final vector is obtained with max-pooling, and passed to a two layer neural network that outputs the parameter $\lambda$ of Poisson distribution.

\textbf{Point cloud classification.} We use two layer 1-dim. convolutional neural network (filter size: 64; kernel size: 3), with batch normalization, that outputs a feature vector of shape (512, 64). We process this with an ordinary differential equation, parametrized either with a deep set or attention.
We use max-pooling, and a two layer feedforward network that outputs class logits. During training we use data augmentation, randomly rotate the object around the z-axis with maximum rotation of $15^{\circ}$, and scale randomly from 80\% to 125\%, following previous works.

\section{Ablation studies}

\subsection{Comparison to \citet{kohler2020equivariant}}

In addition to the main experiments in \Secref{sec:experiments}, we compare to the model by \citet{kohler2020equivariant} that also uses an equivariant flow with closed-form trace. The results are shown in Table~\ref{tab:koehler-comparison}. Their model is designed for physics simulations so it has to include rotation and translation symmetry. They achieve this by transforming the points based on their distances using a simple Gaussian kernel. Because of this limitation, it is unable to model the inhomogeneous part of the process as can be seen from the results. To combat this we added additional spline coupling layers but the final results are still worse than for our model. We conclude that our equivariant transformation (e.g., \Eqref{eq:deepset}) is better suited for the problem of modeling sets and achieves better performance than all competing methods while having closed-form trace computation.

\begin{table}
    \centering
    \begin{tabular}{lccc}
        & Mixture & Thomas & Checkins-NY \\
        \midrule
        \citeauthor{kohler2020equivariant}  & -1.0511 & -0.0134 & -0.6368 \\
        +IHP                                & -1.5583 & -0.0148 & -0.7244 \\
    \end{tabular}
    \caption{Gaussian kernel results.}
    \label{tab:koehler-comparison}
\end{table}

\subsection{Inference time}

Having a stochastic trace during training is the only way to train the continuous normalizing flow when we do not have the exact trace calculation. However, for inference, we want to have exact likelihood which requires exact trace. Calculating trace scales quadratically with the size of the set and elements' dimension. \Figref{fig:trace_speed} showed that our method is much faster and scales well with big sets. Here, we report the total time to calculate the loss on the held-out test set after training is finished. \Figref{fig:speed_comparison} shows that our approach is often an order of magnitude or more faster than computing the exact trace.

\begin{figure}[t]
    \centering
    \includegraphics{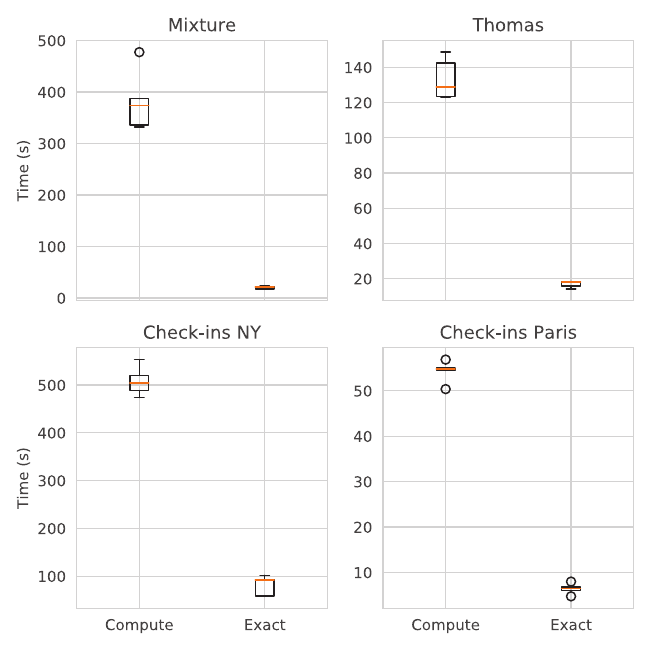}
    \caption{Test time comparison.}
    \label{fig:speed_comparison}
\end{figure}

\begin{figure}[t]
    \centering
    \includegraphics{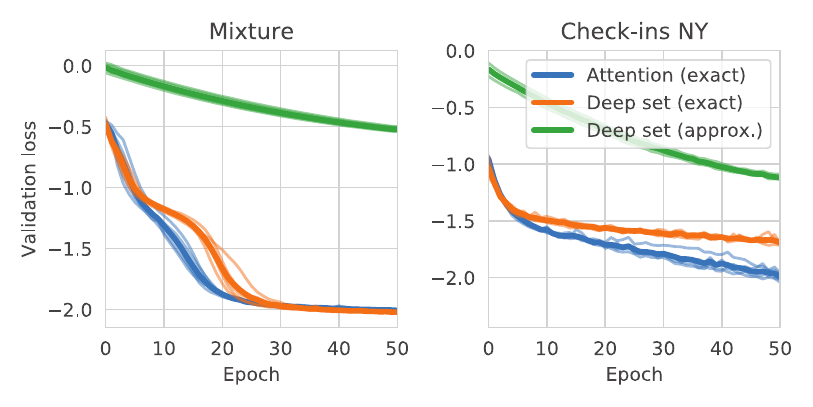}
    \caption{Convergence comparison.}
    \label{fig:convergence}
\end{figure}

\subsection{Convergence}

One of the benefits of our approach is that having an exact trace reduces the variance during training which in turn gives better convergence. We already saw this in \Figref{fig:evaluations_comparison}, for big point cloud sets. Here, in \Figref{fig:convergence}, we show this on two point process datasets. Both deep set and attention layers with exact trace perform better than the one with approximate trace calculation, in the early stages of training.

\subsection{Volume preserving flow}

\begin{table}[b]
    \centering
    \begin{tabular}{lccc}
        & IHP & Zero trace & Zero trace + IHP \\
        \midrule
        Mixture  & \textbf{-2.06} & -0.02 & -2.04 \\
        Mat\'ern & -0.21 & -0.11 & \textbf{-0.45} \\
        Thomas   & -0.01 & -0.15 & \textbf{-0.22} \\
    \end{tabular}
    \caption{Comparison of using a CNF model with and without additional coupling layers.}
    \label{tab:zero_trace_results}
\end{table}

A volume preserving model has only interactions between the points in its drift function so it is unable to model the inhomogeneous part of the process. That is the reason why we add additional coupling layers, same as those in IHP, along with the CNF layer. Here, we compare the two versions of the model, one with 5 of these layers and other without any. We additionally include IHP into comparison. The results in Table~\ref{tab:zero_trace_results} show that we need the inhomogeneous layers to achieve competitive results. They are on par with IHP on a synthetic inhomogeneous dataset. It performs much better on datasets with interactions. A version with no coupling layers has a better test loss score than IHP on Thomas dataset.

\subsection{Induced attention}

Table~\ref{tab:induced_results} shows the comparison between using induced (number of inducing points: $m = 8$) and regular self-attention. We can see that induced attention usually performs worse. We did not tune the hyperparameter $m$ so the results can be different for the optimal value.

\begin{table}
    \centering
    \begin{tabular}{lcc}
        & Induced Attention & Self Attention \\
        \midrule
        Mat\'ern & -0.99 $\pm$ 0.00 & \textbf{-1.03 $\pm$ 0.01}\\
        Mixture  & \textbf{-2.02 $\pm$ 0.03} & -1.92 $\pm$ 0.06 \\
        Thomas   & -0.29 $\pm$ 0.09 & \textbf{-0.55 $\pm$ 0.00} \\
    \end{tabular}
    \caption{Comparison of using induced and regular self-attention.}
    \label{tab:induced_results}
\end{table}

\end{document}